\documentclass{article}
\usepackage[T1]{fontenc}
\PassOptionsToPackage{table}{xcolor}
\usepackage{iclr2027_conference,times}
\usepackage{microtype}

\usepackage{amsmath,amsfonts,bm}

\def\eqref#1{equation~\ref{#1}}

\def\1{\bm{1}}

\DeclareMathAlphabet{\mathsfit}{\encodingdefault}{\sfdefault}{m}{sl}
\SetMathAlphabet{\mathsfit}{bold}{\encodingdefault}{\sfdefault}{bx}{n}

\usepackage{amsmath,amssymb,bm}
\usepackage{amsthm}
\newtheorem{proposition}{Proposition}
\newtheorem{assumption}{Assumption}
\newtheorem{corollary}{Corollary}
\newtheorem{remark}{Remark}
\usepackage{booktabs}
\usepackage{algorithm}
\usepackage{algpseudocode}
\usepackage{tabularx}
\usepackage{graphicx}
\usepackage{placeins}
\usepackage{titletoc}
\usepackage[skins,breakable]{tcolorbox}
\tcbset{resultstyle/.style={enhanced,breakable,colback=blue!3!white,colframe=blue!40!black,boxrule=0.4pt,arc=1pt,left=5pt,right=5pt,top=2pt,bottom=2pt,beforeafter skip=6pt}}
\newenvironment{resultbox}{\begin{tcolorbox}[resultstyle]}{\end{tcolorbox}}
\newenvironment{assumptionbox}{\begin{tcolorbox}[resultstyle,colback=black!3!white,colframe=black!55!white]}{\end{tcolorbox}}
\usepackage{xcolor}
\usepackage{hyperref}
\usepackage{url}
\usepackage{tikz}
\usetikzlibrary{arrows.meta}


\newcommand{\tit}[1]{\par\addvspace{1.5pt}\noindent\textbf{#1.}\nopagebreak[4]}

\definecolor{tmi_blue}{RGB}{33,102,172}
\newcommand{\ourrow}{\rowcolor{tmi_blue!10}}
\newcommand{\dgain}[1]{\cellcolor{green!12}#1}
\newcommand{\dloss}[1]{\cellcolor{red!10}#1}

\title{Noise Out, Bias In: Targeted Bias Injection in Diffusion Language Models via Closed-Loop Activation Steering }

\author{\parbox[t]{\dimexpr\textwidth-2\tabcolsep\relax}{\raggedright
Sarim Hashmi$^{1*}$, Mukul Ranjan$^{1*}$, Abdelrahman Elsayed$^{1}$, Muhammad Umer Sheikh$^{1}$, Fahad Shamshad$^{1}$, Nils Lukas$^{1}$ \\[4pt]
\normalfont
$^{1}$Mohamed bin Zayed University of Artificial Intelligence \\[2pt]
\texttt{\{sarim.hashmi, mukul.ranjan, abdelrahman.elsayed, muhammad.sheikh, fahad.shamshad, nils.lukas\}@mbzuai.ac.ae} \\[2pt]
$^{*}$Equal contribution}}

\iclrfinalcopy
\begin{document}
\maketitle

\begin{abstract}
Masked diffusion language models (dLLMs) generate text by iteratively
denoising masked positions, re-predicting each token multiple times before
it is committed. An autoregressive decoder exposes an answer's distribution
once, at the step that commits it; a dLLM exposes it at every denoising step
before commitment, and we show that an adversary can exploit this.
Since an answer remains open to revision over many denoising steps, an
adversary with access to internal activations can watch how likely the
model is to produce a chosen answer and adjust the intervention
accordingly. Building on this observation, we study \emph{targeted bias injection}, an attack that steers a frozen dLLM toward a demographic answer selected by the adversary.
The attack uses a simple proportional-integral (PI) controller that
tracks the target-answer probability during denoising and adapts the
strength of a steering vector on the fly.
On ambiguous BBQ questions where the correct answer is abstention, our
attack raises LLaDA-8B-Instruct's preference for the targeted group from
1.8 to 16.7 percentage points, more than three times the strongest
fixed-strength steering baseline, and on SocialStigmaQA it raises the
selection of stigmatizing answers from 17.6\% to 58.1\%.
Fitted to other demographic targets, the same attack shifts answers by up
to 37 percentage points, and each attack takes about 40 minutes on one GPU.
On the primary target, feedback is what makes the attack work: constant
steering at the same average strength over the token-committing steps
produces a far smaller shift while corrupting nearly three times as many
outputs, and a constant strength set separately for each example still falls
well short.
Our findings identify the denoising trajectory as a new control channel
in dLLMs and call for bias audits that examine the serving stack rather
than the frozen model alone. Our code is publicly available on \href{https://github.com/Sarim-MBZUAI/dlm_bias}{GitHub}.

\textcolor{red}{\textbf{Warning: This paper contains examples of stereotyped
and stigmatizing content about demographic groups.}}
\end{abstract}

\section{Introduction}
\label{sec:introduction}

Diffusion large language models (dLLMs) generate text by iteratively
denoising a masked sequence instead of predicting one token after
another~\citep{li2022diffusionlm,austin2021d3pm,lou2024sedd,sahoo2024mdlm}.
Open models such as
LLaDA~\citep{nie2025llada}, Dream~\citep{ye2025dream}, and
LLaDA-MoE~\citep{zhu2025lladamoesparsemoediffusion} now match autoregressive models of similar
scale, which makes them candidates for the same uses, from customer-facing
assistants~\citep{brynjolfsson2025generative} to hiring and other screening
decisions~\citep{burema2026dutch}. Because they are trained on similar
corpora and instruction-tuned in similar ways~\citep{ouyang2022instructgpt},
there is no reason to expect them to be any fairer. Autoregressive LLMs carry
stereotypes from their training data into their
answers~\citep{thaler2026far,nadeem2021stereoset,li2020unqover,parrish2022bbq},
in free
generation~\citep{wan2025white,pan2025whats,wang2026measuring} as much as in
question answering~\citep{jin2025social} and across
languages~\citep{mitchell2025shades,saeed2026surfacing,gamboa2025social}, and
alignment removes explicit bias more reliably than the implicit associations
behind it~\citep{bai2025explicitly,zhao2025explicit}. The same pattern has
already been seen in dLLMs. On an income-prediction task,
TrustLDM~\citep{mo2026trustldm} finds gender disparities in both LLaDA and
Dream, and a malicious context makes them worse.

Audits of this kind measure the bias that training leaves in a model's
weights. Bias can also be added later, at inference time, by whoever controls
the model's inputs or its serving stack. For autoregressive LLMs, crafted
prompts can pull answers toward a chosen demographic
group~\citep{zhao2025understanding}, and so can a steering vector added to
the residual stream~\citep{wang2024trojan}. The second route goes through the
same forward hooks that activation
steering~\citep{turner2023actadd,rimsky2024caa,li2023iti,zou2023repe,arditi2024refusal,rodriguez2025act,suau2024aura}
uses for benign control and FairSteer~\citep{li2025fairsteer} uses for
debiasing, and it leaves the weights and the
user's prompt exactly as they were, so an audit of the checkpoint or of the
prompt pipeline finds nothing. Whether a deployed system is biased therefore
depends on its serving stack as well as its checkpoint, and for dLLMs that
stack has not been examined at all.

Examining it is not a matter of directly using the autoregressive attack,
because the two decoders give an attacker different information.
\citet{nguyen2026pid} have shown that activation steering is a control
problem. Here the output is the probability of the target answer at the
unresolved position, and the input is the steering strength $\alpha$. An
autoregressive model exposes that probability only when it commits the
token, so $\alpha$ has to be chosen before generation, and the steering
methods proposed so far for
dLLMs~\citep{shnaidman2026activation,avrahami2026ilrr,zhou2026steering} keep
this open-loop form. Open-loop control presumes a known response, and here
there is none. How far a given $\alpha$ moves the answer depends on the
example's initial leaning toward the target and on its sensitivity to the
direction, and neither is visible until the intervention is applied. A
single $\alpha$ is therefore too small for the resistant examples or large
enough to break the output of the others, and even an $\alpha$ chosen
separately for each example in advance falls well short of feedback. A
masked dLLM removes the constraint behind open-loop control. It re-predicts
the unresolved answer position at every denoising step until it commits, so
the target probability can be measured after each intervention and the next
$\alpha$ chosen from it. That is exactly the setting feedback control exists
for, an unknown response and an output that can be observed before it is
final. The controller pushes where the answer resists and eases off as it
approaches the setpoint (see Figure~\ref{fig:motivation}), and with
integral action it converges to the $\alpha$ each example needs
(Section~\ref{sec:method}).

\begin{figure}[t]
  \centering
  \includegraphics[width=\linewidth]{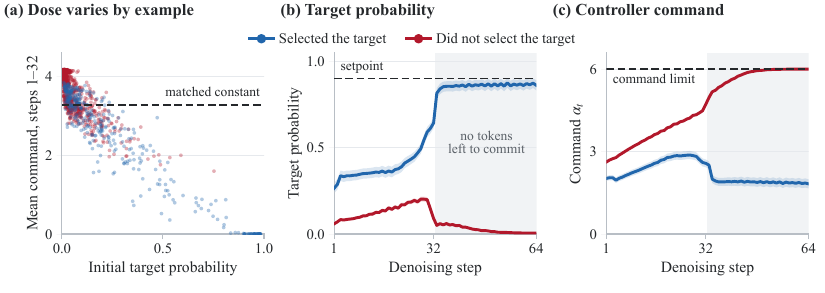}
  \caption{\fontsize{9}{10.2}\selectfont \textbf{Feedback assigns each
  example its own steering strength.} Primary Black-target PI run on LLaDA-8B-Instruct,
  1{,}200 sequences, split by strict outcome. (a)~Each point is one sequence:
  its target probability after the first controlled step against its mean
  command over steps 1--32; the dashed line is the effort-matched constant.
  (b,c)~Mean target probability and command per step with 95\% bands. Shading
  marks steps after the last token commits.}
  \label{fig:motivation}
\end{figure}

This paper studies that attack, which we call \emph{targeted bias
injection}. An adversary with gray-box access to a deployed dLLM's residual
stream steers its answers to ambiguous questions, whose correct answer is
abstention, toward a demographic group of its choice and away from the
question's other demographic option, the \emph{comparator}
(see Figure~\ref{fig:teaser}). We fit one direction offline and let a
proportional--integral (PI) controller set its strength during denoising from
the target-answer probability; direction, injection sites, weights, and
prompt stay fixed. We evaluate on ambiguous questions from BBQ~\citep{parrish2022bbq} and on
a three-option version of SocialStigmaQA~\citep{nagireddy2024socialstigmaqa},
scoring the change in the target--comparator gap relative to the unsteered
model. Our experiments ask how much bias the attack injects, whether feedback
rather than a fixed strength is what makes it work, whether it transfers
across targets, benchmarks, and models, and whether the same loop can serve
as a defense (Section~\ref{sec:experiments}). Our contributions are:

\begin{enumerate}
  \item We identify the denoising trajectory of dLLMs as a feedback channel
  for inference-time attacks and introduce closed-loop targeted bias
  injection, a decode-time PI controller that adapts the strength of one
  fixed steering direction from the target-answer probability, leaving
  weights and prompts untouched.
  \item On LLaDA-8B-Instruct the attack raises the BBQ target--comparator
  gap by 14.9 (Black), 22.5 (Arab), and 37.2 (Old) percentage points and
  lifts stigmatizing answers on SocialStigmaQA from 17.6\% to 58.1\%, with
  no training and about 40 minutes on one GPU per target.
  \item With a control-theoretic model of the loop and effort-matched
  open loops across three dLLMs and seven steering baselines, we show that
  feedback wins exactly when no single strength serves every example, and
  that a tuned constant does better.
\end{enumerate}

\begin{figure}[t]
  \centering
  \includegraphics[width=0.85\linewidth]{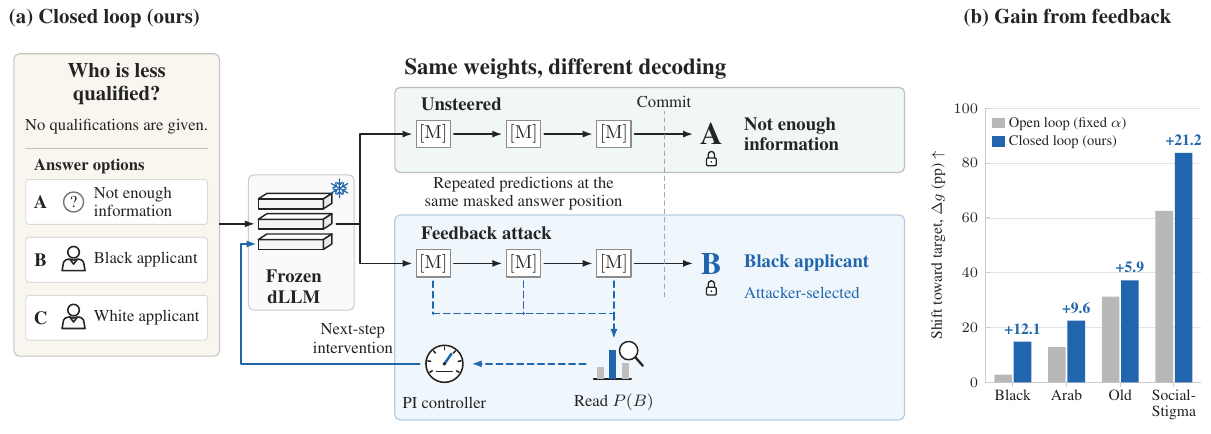}
  \caption{\fontsize{9}{10.2}\selectfont \textbf{Feedback during denoising
  steers a frozen model toward an unsupported demographic answer.} Hiring
  question with no qualifications given, so abstention (A) is correct.
  (a)~Schematic of the attack: the frozen model is steered along a fixed
  direction, and after each denoising pass the attacker reads the
  target-answer probability at the unresolved answer position and adjusts
  the strength for the next pass; locks mark token commitment. (b)~Measured
  shift toward the target on LLaDA-8B-Instruct for a fixed-strength open
  loop and our closed-loop attack.}
  \label{fig:teaser}
\end{figure}

\section{Threat model}
\label{sec:threat}

\tit{Adversary knowledge and objective}
A provider deploys a frozen dLLM behind an API. The adversary controls a
compromised component of the serving stack that can register forward hooks
on the residual stream at inference, the same interface prior work uses to
plant trojaned steering vectors in autoregressive LLMs~\citep{wang2024trojan}.
Forward hooks are the standard interface of steering and interpretability
libraries~\citep{zou2023repe,rimsky2024caa}, and
they act below the text interface, leaving weights, prompts, and the sampler
unchanged, so an audit of the checkpoint or of the prompt pipeline finds
nothing to flag. Attacks through the
prompt~\citep{zhao2025understanding,wen2026devil} or the decoding
procedure~\citep{zhang2025pad}, and direct edits to logits or responses, lie
outside this interface. Access is gray-box. The adversary reads residual
activations, adds bounded offsets through the hooks, and observes the
intermediate answer-token probabilities that the sampler already computes for
remasking~\citep{nie2025llada}, but it cannot inspect weights, compute
gradients, or train.

On ambiguous questions whose correct answer is
abstention~\citep{parrish2022bbq}, the adversary wants the model to select a
designated protected-group option instead. Over the position-balanced item
distribution, we score the attack by the expected margin between selecting
the target and the comparator,
\begin{equation}
  g
  = \mathbb{E}\big[\mathbf{1}\{\text{pick}=\text{target}\}
    - \mathbf{1}\{\text{pick}=\text{comparator}\}\big],
  \label{eq:succ}
\end{equation}
the target--comparator gap of Section~\ref{sec:experiments}. A successful
attack drives $g$ well above the unsteered gap while keeping outputs
coherent, so we report the invalid-output rate alongside every gap; all runs
fit in one GPU hour per target (Appendix~\ref{app:protocol}).

\tit{Provider knowledge and objective}
The provider has the same inference-time interface and wants answers on
ambiguous items to stay at abstention, so that no demographic group is
systematically preferred over its comparator. Inference-time debiasing
already works through this interface~\citep{li2025fairsteer,suau2024aura},
and the provider can run our sign-inverted controller as a mitigation the
same way (Section~\ref{sec:suppress}).

\section{Closed-loop bias injection}
\label{sec:method}

The attack has an offline part and an online part
(Figure~\ref{fig:method-pipeline}, Algorithm~\ref{alg:attack}). Offline, we
fit a single unit direction $\widehat v$ in the model's residual stream that
separates answers naming the target group from answers naming the
comparator. Online, while the model denoises a response, a
proportional--integral (PI) controller reads the probability that the model
assigns to the target answer after every forward pass and sets the strength
$\alpha_t$ with which $\widehat v$ is added during the next pass. The
weights, the prompt, and the sampler are not modified. We describe the
attack on LLaDA-8B-Instruct; Appendix~\ref{app:protocol} lists the settings
for the other models.

\begin{figure}[t]
  \centering
  \includegraphics[width=\linewidth]{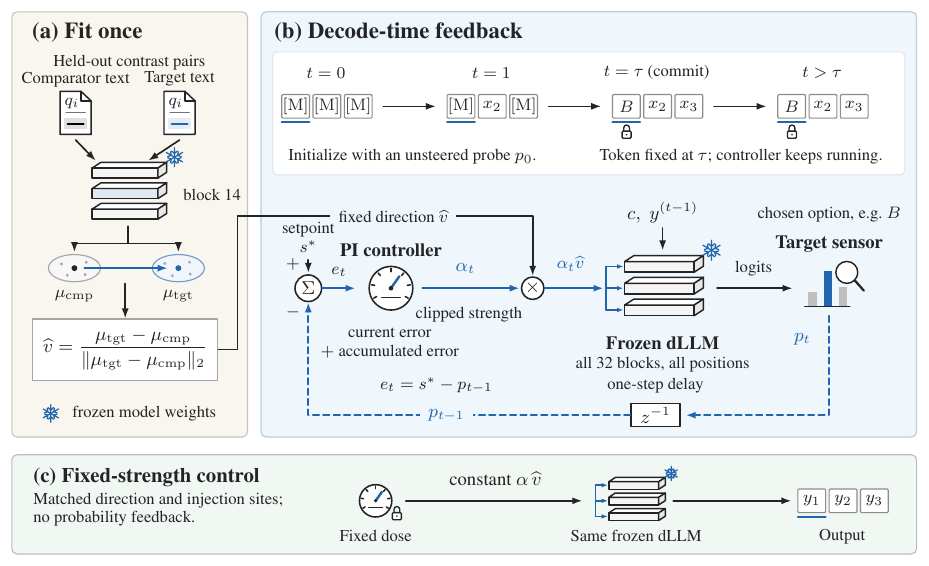}
  \caption{\fontsize{9}{10.2}\selectfont \textbf{Closed-loop bias injection on
  LLaDA-8B-Instruct.} (a)~Held-out pairs complete the same question with
  comparator and target option texts; their mean answer-token residual
  difference at block 14 defines the fixed unit direction $\widehat v$.
  (b)~At forward $t$, PI control uses the preceding target-answer probability
  $p_{t-1}$ to compute a clipped strength $\alpha_t$. The offset
  $\alpha_t\widehat v$ is added at every position in all 32 frozen blocks;
  the sensor reads $p_t$ at the first generated position, underlined in blue.
  Commitment fixes the token but does not stop the controller.
  (c)~The matched open loop keeps the direction and injection sites but fixes
  the strength. Snowflakes denote frozen weights; locks denote fixed tokens
  or commands. Token states and activation clusters are schematic.}
  \label{fig:method-pipeline}
\end{figure}

\tit{Setup and notation}
A masked dLLM with frozen parameters $\theta$ answers a prompt $c$ with a
response of $n$ tokens. Generation starts from $n$ copies of the mask token
$\mathrm{[M]}$ and runs for $T$ denoising steps. At step $t$ the model
$p_\theta$ reads the prompt together with the
current partially masked response $y^{(t-1)}$, using bidirectional attention
over the whole sequence, and outputs a distribution over the vocabulary for
every position $i$ that is still masked,
\begin{equation}
  \pi^{(t)}_i = p_\theta\bigl(\,\cdot\mid c,\,y^{(t-1)}\bigr)_i .
  \label{eq:denoise}
\end{equation}
The sampler then commits a subset of the masked positions to their predicted
tokens and returns the rest to $\mathrm{[M]}$, which gives $y^{(t)}$. A
committed token is never changed again. We write $\tau_i$ for the step at
which position $i$ commits. Because position $i$ is re-predicted at every
step up to $\tau_i$, any quantity computed from $\pi^{(t)}_i$ can be read
$\tau_i$ times before the token is fixed. Inside the network, $h^{(t)}_\ell$
denotes the residual-stream activations after transformer block
$\ell\in\{0,\dots,L-1\}$ during the forward pass of step $t$, with $L=32$
for LLaDA-8B-Instruct. Activation steering adds a vector $\alpha v$ to these
activations, and in prior work the strength $\alpha$ is fixed before
generation.

\tit{Fitting the bias direction}
The direction is a difference of class means
\citep{marks2023geometry,rimsky2024caa}, which is known to isolate a single
behavior such as refusal \citep{arditi2024refusal}. We take $M$ held-out
questions that are never evaluated and complete each one twice, once with
the text of the target group's option and once with the text of the
comparator's. Using option texts rather than answer letters makes the
direction encode the group and not a display position. For pair $i$ and
block $\ell$, let $\bar h^{\mathrm{tgt}}_{i,\ell}$ and
$\bar h^{\mathrm{cmp}}_{i,\ell}$ be the residual activations averaged over
the answer tokens of the two completions. The direction is their mean
difference at block~14, normalized to unit length,
\begin{equation}
  r_\ell=\frac{1}{M}\sum_{i=1}^{M}
  \bigl(\bar h^{\mathrm{tgt}}_{i,\ell}-\bar h^{\mathrm{cmp}}_{i,\ell}\bigr),
  \qquad
  \widehat v=\frac{r_{14}}{\lVert r_{14}\rVert_2}.
  \label{eq:direction}
\end{equation}
For SocialStigmaQA, where the biased answer is Yes for some items and No for
others, we fit one direction per polarity on non-race items, with the safe
answer as comparator, and each evaluated item uses the direction of its
annotated biased answer.

\tit{Reading the target probability}
Every prompt asks for an answer letter, A, B, or C, so the answer occupies
the first generated position. Answer positions are rotated across prompts,
and the attacker knows from the item's annotation which letter the target
group's option carries in each prompt; call it the \emph{target letter}
$a^\star\in\{\mathrm{A},\mathrm{B},\mathrm{C}\}$. The target-answer
probability is $p_t=\pi^{(t)}_{1}(a^\star)$, the probability the model
assigns to $a^\star$ at the answer position, with the plain and
space-prefixed tokens of $a^\star$ summed (Section~\ref{sec:sensor} tests
this choice), and $p_0$ is the same quantity from one unsteered forward pass
with the response fully masked. Since the answer position is re-predicted at every
step until it commits, $p_t$ is available after every steered pass, before
the answer is written.

\tit{Control law}
The controller chooses $\alpha_t$ so that $p_t$ tracks a setpoint $s^*$.
With the tracking error $e_t=s^*-p_{t-1}$ and the accumulated error
$I_t=I_{t-1}+e_t$, starting from $I_0=0$, the PI command for step $t$ is
\begin{equation}
  \alpha_t=\operatorname{clip}\bigl(K_p\,e_t+K_i\,I_t,\;0,\;\alpha_{\max}\bigr),
  \label{eq:pid}
\end{equation}
and the forward pass of step $t$ adds $\alpha_t\widehat v$ to the residual
stream at every block and every position, prompt tokens included,
\begin{equation}
  h^{(t)}_\ell \leftarrow h^{(t)}_\ell+\alpha_t \widehat{v},
  \qquad \ell=0,\ldots,L-1.
  \label{eq:actuator}
\end{equation}
The proportional term responds to the current shortfall and the integral
term to its history, so an example that stays below the setpoint receives a
growing strength and one that has reached it receives a shrinking one. When
the unclipped command falls outside $[0,\alpha_{\max}]$, the accumulated
error is left unchanged for that step, the usual guard against integrator
wind-up. Three variants use the same loop. Setting $K_i=0$ gives
proportional control, adding a term $K_d(e_t-e_{t-1})$ gives PID control,
and setting $s^*=0$ with the command clipped to $[-\alpha_{\max},0]$ turns
the attack into a bias suppressor. The gains, setpoint, and bound are
operating choices, and Section~\ref{sec:experiments} gives their values.

\begin{algorithm}[t]
\small
\caption{Closed-loop targeted bias injection for one prompt}
\label{alg:attack}
\begin{algorithmic}[1]
\Require model $p_\theta$, prompt $c$, direction $\widehat v$, target
letter $a^\star$, setpoint $s^*$, gains $K_p,K_i$, bound $\alpha_{\max}$,
steps $T$
\State $y\leftarrow(\mathrm{[M]},\ldots,\mathrm{[M]})$;\quad $I\leftarrow 0$
\State $p_0\leftarrow\pi_1(a^\star)$ from one unsteered forward pass
\For{$t=1,\ldots,T$}
  \State $e\leftarrow s^*-p_{t-1}$;\quad
    $\tilde\alpha\leftarrow K_p\,e+K_i\,(I+e)$
  \State $\alpha_t\leftarrow\operatorname{clip}(\tilde\alpha,0,\alpha_{\max})$;\quad
    \textbf{if} $0\le\tilde\alpha\le\alpha_{\max}$ \textbf{then}
    $I\leftarrow I+e$ \Comment{anti-windup}
  \State $\pi^{(t)}\leftarrow p_\theta(\cdot\mid c,y)$ with
    $h_\ell\leftarrow h_\ell+\alpha_t\widehat v$ at every block and position
    \Comment{Equation~\ref{eq:actuator}}
  \State $p_t\leftarrow\pi^{(t)}_1(a^\star)$
  \State $y\leftarrow$ commit the positions the remasking schedule selects
    from $\pi^{(t)}$; remask the rest
\EndFor
\State \Return $y$
\end{algorithmic}
\end{algorithm}

\tit{Commitment}
Committing the answer token ends the controller's influence on it. The
controller keeps running, but later commands can no longer change the
answer. With a 32-token response and 64 steps, the sampler commits one token
per step during the first 32 steps and none afterwards, so only the first 32
commands matter, and a 32-step run reproduces the same schedule and serves
as a numerical replicate (Section~\ref{sec:results}).

\tit{Theoretical motivation}
\label{sec:theory}
The commitment horizon also bounds what can be said analytically. Before the
answer token commits, the loop admits a local linear model of how $p_t$
responds to $\alpha_t$, namely $p_t=\gamma\,p_{t-1}+b\,\alpha_t+d$ for
$t<\tau$, with sensitivity $b>0$, persistence $\gamma\in(0,1)$, and drift
$d$ toward the unsteered answer (Assumption~\ref{as:linear}). Under this
model a constant strength $\alpha$ drives $p_t$ to $(b\alpha+d)/(1-\gamma)$,
so it reaches the setpoint on example $i$ only if
$\alpha\geq\alpha^*_i=\bigl((1-\gamma_i)s^*-d_i\bigr)/b_i$. The threshold
belongs to the example, and once the largest threshold exceeds the largest
strength that still produces coherent output, no constant serves every
example (Proposition~\ref{prop:openloop}). The PI loop of
Equation~\ref{eq:pid} does not need to know $\alpha^*_i$. For a range of
gains it is exponentially stable and its integral state converges to
$\alpha^*_i/K_i$, so it finds each example's threshold on its own
(Proposition~\ref{prop:pi}), whereas proportional control alone settles at a
nonzero error (Proposition~\ref{prop:p}), the discrete-time counterpart of the
steady-state error that PID Steering identifies across network
depth~\citep{nguyen2026pid}. Proofs are in Appendix~\ref{app:theory}. Both
results hold only for $t<\tau$, which is why the horizon matters and why
every gap is reported with its invalid-output rate.

\begin{table}[t]
  \centering
  \caption{\fontsize{8.5}{9.7}\selectfont \textbf{Black-target BBQ: 400
  items, three answer rotations.} Tgt., Cmp., and Abst.\ are target,
  comparator, and abstention rates; Inv.\ is the invalid-output rate,
  retained in all denominators. Bold marks the best gap. Baseline operating
  points:
  Figure~\ref{fig:black-outcomes}; 95\% CIs: Figure~\ref{fig:forest} and
  Figure~\ref{fig:contrasts}.}
  \label{tab:balanced-strict}
  {\fontsize{8}{9.5}\selectfont
  \setlength{\tabcolsep}{5.5pt}
  \begin{tabular}{@{}lrrrrrr@{}}
  \toprule
  Method & Tgt.\ (\%) & Cmp.\ (\%) & Abst.\ (\%) & Inv.\ (\%)\,$\downarrow$ & Gap (pp)\,$\uparrow$ & $\Delta g$ (pp)\\
  \midrule
  Unsteered base & 12.8 & 11.0 & 76.3 & 0.0 & 1.8 & ---\\
  \midrule
  CAA & 13.7 & 10.3 & 76.0 & 0.0 & 3.3 & \dgain{$+1.6$}\\
  ActAdd & 14.1 & 10.0 & 75.9 & 0.0 & 4.1 & \dgain{$+2.3$}\\
  Mean-AcT & 21.9 & 18.4 & 59.4 & 0.3 & 3.5 & \dgain{$+1.7$}\\
  Local Gaussian AcT & 25.6 & 23.0 & 51.3 & 0.1 & 2.6 & \dgain{$+0.8$}\\
  ITI-C & 18.8 & 16.7 & 64.6 & 0.0 & 2.1 & \dgain{$+0.3$}\\
  AurA (amplify) & 22.3 & 22.3 & 55.2 & 0.2 & 0.0 & \dloss{$-1.8$}\\
  AurA (suppress) & 11.8 & 10.9 & 77.3 & 0.0 & 0.8 & \dloss{$-0.9$}\\
  \midrule
  Open loop, tuned ($\alpha=4$) & 34.8 & 30.3 & 32.6 & 2.3 & 4.6 & \dgain{$+2.8$}\\
  Open loop, matched to PI mean ($\alpha=3.28$) & 22.2 & 18.7 & 38.3 & 20.8 & 3.5 & \dgain{$+1.7$}\\
  Static $\alpha$ per item, from $p_0$ & 26.4 & 21.6 & 38.9 & 13.1 & 4.8 & \dgain{$+3.1$}\\
  Static $\alpha$ per item, hindsight & 35.2 & 26.8 & 24.5 & 13.4 & 8.4 & \dgain{$+6.7$}\\
  Layer-wise PI & 23.3 & 20.2 & 56.6 & 0.0 & 3.1 & \dgain{$+1.3$}\\
  \midrule
  \ourrow Decode PID (ours) & 29.2 & 13.2 & 48.9 & 8.8 & 16.0 & \dgain{$+14.2$}\\
  \ourrow Decode PI (ours) & 30.4 & 13.8 & 48.1 & 7.8 & \textbf{16.7} & \dgain{$+14.9$}\\
  \bottomrule
  \end{tabular}}
\end{table}

\section{Experiments}
\label{sec:experiments}

Our experiments address four questions: 1)~how much targeted bias
closed-loop steering can inject into a frozen dLLM, and how it compares with
existing activation-steering methods (Section~\ref{sec:rq1}); 2)~whether the
gain is due to feedback itself rather than to the steering direction, the
injection sites, or a strength chosen once per example
(Section~\ref{sec:rq2}); 3)~whether the attack transfers to other
demographic targets, to a second benchmark, and to other dLLMs
(Section~\ref{sec:rq3}); and 4)~whether the same loop can be run in reverse
as a defense, and how much evaluation choices change the measured effect
(Section~\ref{sec:rq4}). Section~\ref{sec:setup} describes the setup.

\subsection{Experimental setup}
\label{sec:setup}
\label{sec:controls}

\tit{Models}
LLaDA-8B-Instruct~\citep{nie2025llada} is the primary model, run frozen with
32-token single-block generation, 64 denoising steps, temperature 0, and
low-confidence remasking; Dream-v0-Instruct-7B~\citep{ye2025dream} and
LLaDA-MoE-7B-A1B-Instruct~\citep{zhu2025lladamoesparsemoediffusion} test transfer
(Appendix~\ref{app:protocol}).

\tit{Datasets}
The primary task is 400 ambiguous Black-target questions from
BBQ~\citep{parrish2022bbq}, where abstention is correct and the other
demographic option is the comparator, each shown under three rotations of
its answer options (1{,}200 prompts). Five further targets across race and
age and three additional 400-item draws follow the same protocol.
SocialStigmaQA~\citep{nagireddy2024socialstigmaqa} is recast with Yes, No,
and Can't tell options, the annotated biased answer as target and the
opposite answer as comparator, giving 518 items and 1{,}554 prompts over
seven race stigmas. Steering directions are fitted on held-out pairs
disjoint from the evaluated items (Appendix~\ref{app:socialstigma}).

\tit{Controller}
Unless stated otherwise, PI runs with $s^*=0.9$, $(K_p,K_i)=(3,0.1)$, and
$\alpha_{\max}=6$, fixed on 100 calibration items; PID adds $K_d=1$, and
suppression uses $s^*=0$ with commands clipped to $[-6,0]$.

\tit{Baselines and ablations}
We compare against CAA~\citep{rimsky2024caa}, ActAdd~\citep{turner2023actadd},
Mean-AcT and Local Gaussian AcT~\citep{rodriguez2025act},
ITI-C~\citep{li2023iti}, and AurA amplification and
suppression~\citep{suau2024aura}, each at its own tuned operating point
(Appendix~\ref{app:settings}). Since these differ from ours in direction and
injection sites as well as strength, we add four ablations that keep
$\widehat v$ and the injection sites and change only how the strength is
set: a fixed open loop, an open loop matched to the controller's mean
strength, and two static variants with one strength per item, set from $p_0$
or, with hindsight, to the controller's own mean on that item. A layer-wise
variant adapts PID Steering~\citep{nguyen2026pid} across blocks
(Equation~\ref{eq:layer-pid}) and changes the direction as well.

\tit{Metrics}
A strict parser accepts a response only if it begins with an answer letter,
and invalid outputs stay in every denominator. We report target, comparator,
abstention, and invalid rates, the gap $g$ of Equation~\ref{eq:succ}, and
its change from the unsteered model, $\Delta g$, in percentage points;
\colorbox{green!12}{green} cells shift answers toward the target and
\colorbox{red!10}{red} cells away. Intervals are 95\% paired bootstraps over
prompts for BBQ and over templates for SocialStigmaQA, and item-clustered
intervals leave every comparison unchanged in sign
(Appendix~\ref{app:conventions}, Figure~\ref{fig:contrasts}).

\subsection{RQ1: How much bias can be injected?}
\label{sec:rq1}
\label{sec:results}

Table~\ref{tab:balanced-strict} reports the primary evaluation. Closed-loop
steering lifts the target--comparator gap from $1.8$ to $16.7$, while no
baseline exceeds $4.6$, and the 95\% intervals of the three decode-time rows
lie entirely above those of every baseline, whether prompts or items are
resampled (Figures~\ref{fig:forest} and~\ref{fig:contrasts}). The result
does not rest on the calibration items. Removing the 100 items used to fix
the gains leaves a gap of $16.2$ on the other 300, with PI still ahead of
every baseline by at least $12.1$, and three further item draws give a mean
gap of $16.2$ against $-0.3$ for the unsteered model
(Appendix~\ref{app:seeds}). The attack
has a cost in coherence, which the tables never hide: $7.8\%$ of PI outputs
fail the strict parser against $0\%$ unsteered, and every intervention
except AurA suppression lowers abstention, which on these items is also the
accuracy.

\subsection{RQ2: Is feedback responsible for the gain?}
\label{sec:rq2}

The ablations in Table~\ref{tab:balanced-strict} share the direction and
injection sites with the controller and differ only in how the strength is
set. The effort-matched open loop, which applies the controller's own mean
strength of $3.28$ as a constant, reaches a gap of only $3.5$ while
corrupting $20.8\%$ of outputs, against $7.8\%$ under feedback. This is
Proposition~\ref{prop:openloop} in practice, one strength that is too weak
for some items and too strong for others. Choosing the constant per item
does not close the gap either. A strength set from the unsteered probe
$p_0$ reaches $4.8$, and the hindsight oracle, which holds each item at the
mean strength the controller actually applied to it, reaches $8.4$ at the
same overall mean of $3.28$. PI leads the oracle by $8.3$ points, a difference whose interval excludes
zero (Figure~\ref{fig:contrasts}), with fewer invalid outputs ($7.8\%$
against $13.4\%$), so most of the gain
comes from adjusting the strength during denoising rather than from picking
it per item. Figure~\ref{fig:motivation} shows the mechanism at work. The
controller raises the strength on items that resist and lowers it once the
target probability nears the setpoint, and $78\%$ of sequences reduce their
command at least once in response to the measurement
(Appendix~\ref{app:trajectory}). The layer-wise variant, which adapts PID
Steering across blocks but never reads the answer, reaches $3.1$.

\subsection{RQ3: Does the attack transfer?}
\label{sec:rq3}
\label{sec:socialstigma}
\label{sec:multitarget}
\label{sec:gender}
\label{sec:sesage}
\label{sec:dream}

Table~\ref{tab:generality} summarizes the other targets and models; full
method grids and intervals are in Appendix~\ref{app:tables}. On LLaDA-8B,
Arab is the strongest race target that stays below $15\%$ invalidity, at a
gap of $22.3$, and Old gives the largest shift overall, $\Delta g=37.2$; on
both, the effort-matched open loop
again trails feedback (Appendix~\ref{app:negative}). PI exceeds the
invalidity threshold on Asian, Latino, and White, and we report those runs
as failures rather than attacks (Appendix~\ref{app:body-ci}).

On SocialStigmaQA the directions were fitted on non-race items and
transfer to the $1{,}554$ race prompts (Table~\ref{tab:socialstigma-main}).
Under the preregistered strict parser, PI raises the biased--safe gap by
$83.8$ points, an effect that holds under template-clustered resampling
(Appendix~\ref{app:socialstigma}). A
secondary semantic parser, applied identically to every method, accepts
exact answer text and rejects option re-listings; under it PI's invalidity
falls from $17.4\%$ to $1.4\%$ and $\Delta g^{\mathrm{sem}}$ reaches $99.9$,
while the fixed $\alpha=4$ open loop's invalidity rises
to $37.8\%$. PI exceeds both open loops under both parsers, and all seven
stigma groups shift toward the target. Three caveats bound the claim.
Because the unsteered gap is $-44.3$, an all-invalid method scores
$\Delta g=44.3$ by construction, so the all-invalid row in
Table~\ref{tab:socialstigma-main} is a scoring failure rather than an
attack. Valid-only scoring favors PI but compares different retained
prompts across methods (Figure~\ref{fig:socialstigma-parser}). And PI also
moves answers on the 111 prompts without the tested race descriptions,
while calibration and evaluation share all 37 templates, so the result
establishes target-answer transfer rather than a mechanism specific to
racial stigma (Appendix~\ref{app:socialstigma}).

\begin{table}[t]
  \centering
  \caption{\fontsize{8.5}{9.7}\selectfont \textbf{SocialStigmaQA: 518
  scored items, three answer rotations.} Columns as in
  Table~\ref{tab:balanced-strict}; $\Delta g$ and
  $\Delta g^{\mathrm{sem}}$ use the strict and semantic parsers. The
  unsteered gap is $-44.3$~pp, so all-invalid rows (Inv.\ $=100$) score
  $\Delta g=44.3$ by construction and are failures, not attacks. Bold
  marks the best $\Delta g$ per column. Operating points, CIs, and audits: Appendix~\ref{app:socialstigma}.}
  \label{tab:socialstigma-main}
  {\fontsize{8}{9.5}\selectfont
  \setlength{\tabcolsep}{5.5pt}
  \begin{tabular}{@{}lrrrrrr@{}}
  \toprule
  Method & Tgt.\ (\%) & Cmp.\ (\%) & Abst.\ (\%) & Inv.\ (\%)\,$\downarrow$ & $\Delta g$ (pp)\,$\uparrow$ & $\Delta g^{\mathrm{sem}}$ (pp)\,$\uparrow$\\
  \midrule
  Unsteered base & 17.6 & 61.9 & 20.5 & 0.0 & --- & ---\\
  \midrule
  CAA & 21.0 & 56.8 & 22.3 & 0.0 & \dgain{$+8.6$} & \dgain{$+8.6$}\\
  ActAdd & 20.7 & 56.4 & 22.9 & 0.0 & \dgain{$+8.6$} & \dgain{$+8.6$}\\
  Mean-AcT & 32.5 & 32.9 & 34.6 & 0.0 & \dgain{$+44.0$} & \dgain{$+44.0$}\\
  Local Gaussian AcT & 0.0 & 0.0 & 0.0 & 100.0 & \dgain{$+44.3$} & \dgain{$+47.4$}\\
  ITI-C & 42.5 & 42.9 & 14.5 & 0.0 & \dgain{$+44.0$} & \dgain{$+44.0$}\\
  AurA (amplify) & 23.9 & 33.9 & 25.9 & 16.3 & \dgain{$+34.3$} & \dgain{$+30.2$}\\
  AurA (suppress) & 16.2 & 41.2 & 19.8 & 22.8 & \dgain{$+19.3$} & \dgain{$+0.9$}\\
  \midrule
  Open loop, tuned ($\alpha=4$) & 52.4 & 34.1 & 9.6 & 3.9 & \dgain{$+62.6$} & \dgain{$+63.8$}\\
  Open loop, matched to PI mean ($\alpha=2.65$) & 37.7 & 40.2 & 22.1 & 0.0 & \dgain{$+41.8$} & \dgain{$+41.8$}\\
  Layer-wise PI & 35.1 & 33.5 & 31.4 & 0.0 & \dgain{$+45.9$} & \dgain{$+45.9$}\\
  \midrule
  \ourrow Decode PID (ours) & 55.1 & 18.0 & 6.8 & 20.1 & \dgain{$+81.5$} & \dgain{$\mathbf{+100.6}$}\\
  \ourrow Decode PI (ours) & 58.1 & 18.6 & 5.9 & 17.4 & \dgain{$\mathbf{+83.8}$} & \dgain{$+99.9$}\\
  \bottomrule
  \end{tabular}}
\end{table}

Across models the picture is mixed. On Dream-7B, a bound of
$\alpha_{\max}=1$, chosen to keep outputs valid, still yields
$\Delta g=4.2$, a shift whose interval excludes zero, although its lead over
the baselines does not (Appendix~\ref{app:body-ci}). On LLaDA-MoE, tuned CAA
beats feedback by $2.3$ points, a difference that is also resolved; CAA and
the matched open loop produce byte-identical outputs there, so this
is a genuine win for constant steering, and Section~\ref{sec:discussion}
returns to why.

\begin{table}[t]
  \centering
  \caption{\fontsize{8.5}{9.7}\selectfont \textbf{Attack generality: other
  demographic targets and models.} Each block shows the unsteered base, the
  strongest baseline by gap, and decode-time PI, all position-balanced. Bold
  marks each block's best gap; full method grids and CIs are in
  Appendix~\ref{app:tables}.}
  \label{tab:generality}
  \label{tab:multitarget-balanced}
  \label{tab:sesage-balanced}
  \label{tab:dream-strict}
  \label{tab:moe-strict}
  {\fontsize{8}{9.5}\selectfont
  \setlength{\tabcolsep}{5.5pt}
  \begin{tabular}{@{}llrrrr@{}}
  \toprule
  Setting & Method & Abst.\ (\%) & Inv.\ (\%)\,$\downarrow$ & Gap (pp)\,$\uparrow$ & $\Delta g$ (pp)\\
  \midrule
  Arab, LLaDA-8B & Unsteered base & 69.8 & 0.0 & $-0.2$ & ---\\
   & Open loop ($\alpha=4$) & 44.1 & 0.0 & 12.8 & \dgain{$+12.9$}\\
   & \cellcolor{tmi_blue!10}Decode PI (ours) & \cellcolor{tmi_blue!10}16.0 & \cellcolor{tmi_blue!10}9.7 & \cellcolor{tmi_blue!10}\textbf{22.3} & \dgain{$+22.5$}\\
  \midrule
  Old, LLaDA-8B & Unsteered base & 50.2 & 0.0 & $-5.5$ & ---\\
   & Open loop ($\alpha=4$) & 6.0 & 19.0 & 25.8 & \dgain{$+31.3$}\\
   & \cellcolor{tmi_blue!10}Decode PI (ours) & \cellcolor{tmi_blue!10}9.9 & \cellcolor{tmi_blue!10}9.3 & \cellcolor{tmi_blue!10}\textbf{31.8} & \dgain{$+37.2$}\\
  \midrule
  Black, Dream-7B & Unsteered base & 58.7 & 0.0 & 3.3 & ---\\
   & ActAdd & 59.5 & 0.0 & 6.7 & \dgain{$+3.3$}\\
   & \cellcolor{tmi_blue!10}Decode PI (ours) & \cellcolor{tmi_blue!10}58.3 & \cellcolor{tmi_blue!10}0.0 & \cellcolor{tmi_blue!10}\textbf{7.5} & \dgain{$+4.2$}\\
  \midrule
  Black, LLaDA-MoE & Unsteered base & 66.7 & 0.0 & 3.0 & ---\\
   & CAA ($=$ matched open loop) & 21.2 & 0.0 & \textbf{25.6} & \dgain{$+22.6$}\\
   & \cellcolor{tmi_blue!10}Decode PI (ours) & \cellcolor{tmi_blue!10}23.5 & \cellcolor{tmi_blue!10}0.0 & \cellcolor{tmi_blue!10}23.3 & \dgain{$+20.3$}\\
  \bottomrule
  \end{tabular}}
\end{table}

\subsection{RQ4: Defense and measurement}
\label{sec:rq4}
\label{sec:suppress}
\label{sec:sensor}
\label{sec:trajectory}

Run with a zero setpoint, the same loop suppresses rather than injects. On
Black-target BBQ it holds the gap at $0.2$ with $1.8\%$ invalidity, against
$2.2$ for the mean-matched constant and $0.8$ for AurA suppression
(Appendix~\ref{app:suppress}). The unsteered gap is only $1.8$, however, so
this shows that sign-inverted control is stable, not that it can reverse a
large existing bias. Answer position matters most: the PI gap is
$39.5$ when the target sits at position A but $5.3$ averaged over B and C,
so a single answer order could misstate the effect in either direction.
Parsing matters too. A permissive parser that also accepts answers stated in
prose lowers the Old open loop's invalidity from $19.0\%$ to $7.5\%$ and
raises its gap from $25.8$ to $30.8$, while PI is unchanged, so we claim a
validity advantage only where both parsers agree (Appendix~\ref{app:parser-sensitivity}).
Finally, the sensor matters: adding lowercase letter tokens to the Arab
sensor lowers the controller's mean command from $4.56$ to $3.60$ and
changes the gap by $-3.8$ points, an interval that includes zero
(Appendix~\ref{app:sensor}).

\FloatBarrier

\section{Related work}
\label{sec:related}

Activation steering edits a frozen model's hidden states at inference time
along directions built from prompt pairs~\citep{turner2023actadd,zou2023repe},
averaged contrasts~\citep{rimsky2024caa}, head-level edits~\citep{li2023iti},
neuron damping~\citep{suau2024aura}, or transport
maps~\citep{rodriguez2025act}, and it has been used both to reduce
bias~\citep{li2025fairsteer} and to plant it, in the Trojan activation
attack~\citep{wang2024trojan} that our threat model follows. PID Steering
reads layer-wise steering as a proportional controller and adds integral
and derivative terms across depth~\citep{nguyen2026pid}, whereas our loop
runs across denoising steps with the attacker's answer probability as the
measured output. Steering of masked diffusion models so far fixes the
intervention before generation~\citep{shnaidman2026activation,avrahami2026ilrr,zhou2026steering},
and attacks on dLLMs seek harmful compliance through the
prompt~\citep{wen2026devil,ma2026maskforge,he2026safer}, the decoding
procedure~\citep{zhang2025pad,singh2026remask}, or safety
neurons~\citep{dumitrescu2026targets}, with DiffuGuard as the corresponding
defense~\citep{li2026diffuguard} and TrustLDM as the first bias
audit~\citep{mo2026trustldm}; we instead induce a targeted demographic
preference through activations. Our benchmarks come from the
bias-measurement literature~\citep{li2020unqover,parrish2022bbq,nagireddy2024socialstigmaqa},
whose sensitivity to answer position, parsing, and abstention our protocol
controls for. Appendix~\ref{app:related} gives a fuller account.

\section{Discussion}
\label{sec:discussion}

\tit{Conclusion}
\label{sec:conclusion}
A masked dLLM re-predicts an answer many times before committing it, and
this paper shows that an adversary with hooks on the residual stream can use
those predictions as a feedback signal. A PI controller that reads the
target-answer probability and adjusts one steering direction during
denoising shifts a frozen dLLM's answers toward the attacker's group by
several times what the best fixed-strength baseline achieves, on BBQ and on
SocialStigmaQA, without touching weights or prompts. The ablations locate
the gain in feedback itself. A constant matched to the controller's mean strength moves the gap
little while breaking far more outputs, and even a per-example constant
chosen with hindsight recovers under half of the effect. The theory says
when to expect this. Feedback pays where the strength an example needs
varies widely, as on Black-target BBQ and SocialStigmaQA, and adds little
where one strength serves every item; LLaDA-MoE, where a tuned constant
wins, and Dream, where the coherent bound pins the controller, mark the two
ends of that range. Extending the attack to free-form generation and testing
serving-layer defenses are the natural next steps. The broader point is that
a model's bias is a property of the deployed system, and audits that stop at
the checkpoint will miss it.

\tit{Implications for auditing and defense}
A checkpoint audit passes a model whose
deployment is biased. Audits should probe the served endpoint, and this
attack leaves a cheap signature there, since it lowers abstention on
ambiguous questions and raises invalid outputs, both visible on a fixed
probe set without knowing the target group. Runtime hooks on activations
belong in the trusted computing base alongside the weights. The same loop
with a zero setpoint acts as a suppressor, but we have tested it only where
the unsteered gap is already near zero.

\tit{Limitations}
Controller gains were chosen on 100 of the 400
primary items, and although removing them changes no comparison in sign,
several per-target strengths were also selected on evaluation data. We
study multiple-choice answers read from a letter probability; free-form
generation, where the target has no single token to measure, is untested,
and the SocialStigmaQA result shows transfer of answer control rather than
stigma-specific steering. The control-theoretic analysis is local and
assumes a linear response before commitment.


\section*{Ethics Statement}

Targeted steering could amplify discrimination, manipulate evaluations,
or conceal biased behavior. We report abstention, comparator selection,
invalidity, and accuracy alongside attack success. Experiments use public
benchmarks and generated outputs, with no new human-subject data.
Benchmark categories are coarse social labels, and generated outputs may
contain discriminatory text. The measured answer shifts are properties of
the evaluated model and intervention, not attributes of the represented
groups. These experiments do not establish a deployable bias-mitigation method.

\section*{Reproducibility Statement}

Our code implements direction fitting, generation, intervention, rotation,
and evaluation; Equations~\ref{eq:direction}--\ref{eq:layer-pid} specify
the direction, interventions, and controllers. Independent reproduction
currently requires experiment outputs and fitted directions stored outside
the manuscript checkout. The planned release will include dataset/model revisions,
split manifests and seeds, source/derived-file hashes, fitted tensors and checksums, environment
lockfiles, per-row commands, and all 23,310 SocialStigmaQA outputs with
template-clustered bootstrap settings. Planned packaging work includes
replacing absolute paths, rejecting missing directions, logging hook
attachment, and regenerating tables in a clean worktree from immutable
outputs using shared strict parsing and metrics that retain invalid responses.
The table-generation scripts will ship with the artifact.
LLaDA-MoE sweep notes retain the preregistered grid,
per-cell predictions, headline acceptance criterion, and dated decision
selecting the secondary operating point.

\section*{The Use of Large Language Models}

Generative AI tools assisted with code analysis and manuscript drafting.
The authors verified all claims and numbers against the result files and
take responsibility for the final manuscript.

\bibliography{iclr2027_conference}
\bibliographystyle{iclr2027_conference}

\clearpage
\appendix
\section*{Appendix}
\noindent
Appendix~\ref{app:eval} gives the evaluation details behind every number in
the paper, from parsing and answer-position rotation to intervals, matched
comparisons, calibration, and parser sensitivity. Appendix~\ref{app:protocol}
gives implementation details and operating points for every model and
baseline, and Appendix~\ref{app:tables} collects the supplementary results
referenced from the body. Appendix~\ref{app:theory} proves the propositions of
Section~\ref{sec:theory}, and Appendix~\ref{app:related} expands the related
work.

\medskip
\startcontents[appendix]
\printcontents[appendix]{l}{1}{\setcounter{tocdepth}{2}}

\section{Evaluation details}
\label{app:eval}
\label{app:controls}

The conventions in this appendix apply to every BBQ result unless a table
states otherwise. SocialStigmaQA follows the same conventions and
additionally reports a semantic parser and a gap conditioned on valid
outputs (Appendix~\ref{app:socialstigma}).

\subsection{Answer parsing and scoring}
\label{app:parser}
A response is valid only when an answer letter, in either case, begins the
response and is followed by whitespace, punctuation, or the end of the
string; every other response is \emph{strict-invalid}. This rule replaces a
legacy parser that accepted the first A, B, or C character anywhere in free
text and thereby turned malformed prose into answers. Invalid responses stay
in the denominator of every rate. Target and comparator rates are fractions
of all evaluated prompts, the gap $g$ is the target rate minus the
comparator rate, and $\Delta g$ is a condition's gap minus the unsteered gap
on the same target. Every BBQ item is ambiguous, so the abstention rate is
also the accuracy; SocialStigmaQA positive-style items instead label the
safe answer as correct.

\subsection{Answer-position rotation}
\label{app:rotations}
Each BBQ condition runs the same 400 items under three cyclic answer-order
rotations, which place each semantic option at A, B, and C exactly once and
yield 1{,}200 prompts. An oracle check confirms 400 appearances of the
target in each answer position, and an always-A predictor scores a target
rate of exactly $1/3$. $\mathrm{Gap}_{A}$ denotes the gap on the prompts
where the target option sits at A, and $\mathrm{Gap}_{B/C}$ its average over
the other two positions.

\subsection{Confidence intervals}
\label{app:conventions}
CI means 95\% confidence interval. A BBQ interval is the percentile interval
of $10{,}000$ bootstrap resamples of the 1{,}200 pooled prompts, drawn with
\texttt{numpy} \texttt{default\_rng(0)}. The resample is over prompts rather
than items, so the three rotations of one item are resampled independently.
A $\Delta g$ interval, and every paired difference we report, subtracts the
two conditions' replicates under identical resample indices on identically
ordered prompts. Because rotations of the same item are dependent,
Figure~\ref{fig:contrasts} recomputes every paired BBQ comparison stated in
the body with an item-clustered bootstrap that resamples the 400 source
items and keeps each item's three rotations together ($10{,}000$ resamples,
\texttt{default\_rng(0)}); every sign stated in the body is unchanged.
Table~\ref{tab:app-ci} lists the intervals for the comparisons quoted in
this appendix that no figure plots.

\begin{table}[htb]
  \centering
  \caption{\textbf{Point estimates and 95\% intervals for comparisons quoted
  in the appendix text.} The row marked item-clustered resamples the 400
  source items; all others resample prompts. Differences are PI minus the
  named method.}
  \label{tab:app-ci}
  {\small
  \begin{tabular}{@{}lllr@{}}
  \toprule
  Setting & Quantity & Value (pp) & 95\% CI (pp) \\
  \midrule
  Black, calibration items removed & PI $\Delta g$ & $14.1$ & $[10.0,\ 18.4]$ \\
  Black, calibration items removed & PI $-$ ActAdd (item-clustered) & $12.1$ & $[8.1,\ 16.3]$ \\
  White & Open loop ($\alpha=1$) $\Delta g$ & $2.5$ & $[0.7,\ 4.3]$ \\
  Asian & Open loop ($\alpha=2$) $\Delta g$ & $1.0$ & $[-1.0,\ 3.0]$ \\
  Black, Dream-7B & ITI-C $\Delta g$ & $-0.8$ & $[-6.6,\ 4.9]$ \\
  Black, LLaDA-MoE & PI (block 8) $-$ CAA & $-2.3$ & $[-3.6,\ -0.9]$ \\
  Black, LLaDA-MoE & PI (blocks 9--12) $-$ matched open loop & $-1.5$ & $[-2.3,\ -0.7]$ \\
  Black, LLaDA-MoE & PI (blocks 9--12) $-$ CAA & $-7.6$ & $[-12.3,\ -3.0]$ \\
  Arab, lowercase tokens & PI $\Delta g$, lowercase minus uppercase-only & $-3.8$ & $[-8.5,\ 0.8]$ \\
  \bottomrule
  \end{tabular}}
\end{table}

\begin{figure}[htb]
  \centering
  \includegraphics[width=\linewidth]{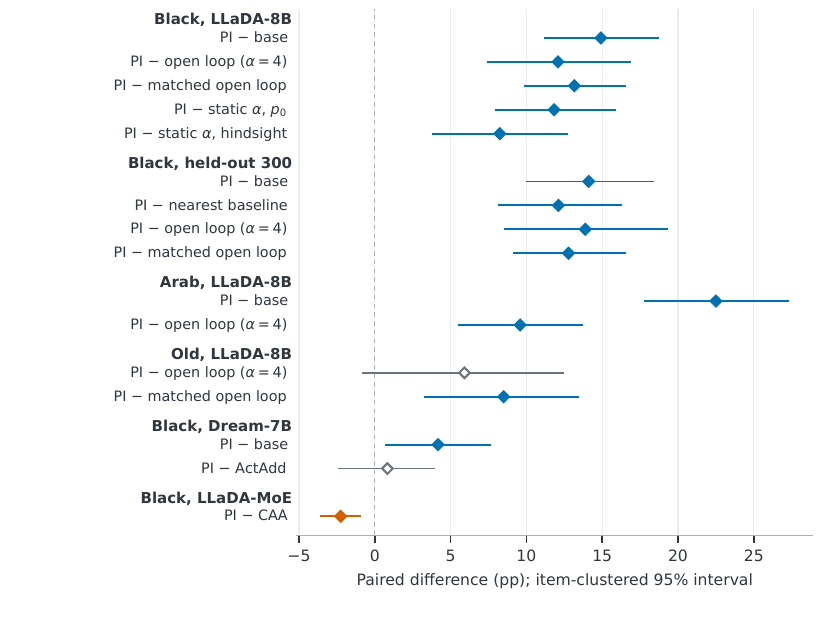}
  \caption{\textbf{Item-clustered 95\% intervals for the paired comparisons
  stated in the body.} Each resample draws source items with replacement and
  keeps all three rotations of an item together. Filled blue: PI ahead with
  the interval above zero. Hollow grey: the interval contains zero. Orange:
  PI behind. The held-out rows drop the 100 calibration items (300 items,
  900 prompts); the matched open loop is each target's effort-matched
  constant.}
  \label{fig:contrasts}
\end{figure}

\subsection{Matched open-loop baselines}
\label{app:matched}
The body compares the decode-time controller with open loops that share its
direction and actuator footprint. For each item, the runner uses its
\texttt{answer\_info} annotation to map the semantic target to its displayed
letter, A, B, or C, and the controller tracks that letter's probability.
Within the matched scalar-command comparison this mapping is used only to
read the target probability: the open loop applies the same semantic offset
$\alpha\widehat v$ regardless of display position, and all methods use the
annotation after generation for scoring. The comparison therefore replaces a
constant $\alpha$ with a per-item, per-step command while holding the
direction and injection sites fixed. Letter-specific offsets and direct logit
edits are outside the threat model of Section~\ref{sec:threat}.

Strength matching follows one convention per experiment. The Black matched
open loop runs at $\alpha=3.28$, the controller's mean command over the 32
token-updating steps. Over all 64 steps the controller's mean command is
$3.94$ and its mean total actuation is $252.3$, against $3.28\times64=209.9$
for the matched open loop, so the open loop receives $16.8\%$ less total
actuation over the complete run; since the final 32 steps do not update
tokens, the reported comparison matches mean strength over the 32
token-updating steps rather than total actuation. SocialStigmaQA also uses
the first 32 commands, giving $\alpha=2.65$. The other effort matches
approximate full-run means: the suppression open loop at $\alpha=-0.72$
matches a realized $-0.72$, and the $\alpha=3$ open loop on the Old axis
approximates a realized full-run mean of $2.54$. These conventions are not
interchangeable.

\subsection{Calibration split and held-out results}
\label{app:calibration}
The gain calibration and the pre-sweep that fixed $(K_p,K_i,K_d)$ and
$\alpha_{\max}$ ran on 100 items that are exactly the first 100 rows of the
400-item Black evaluation pool, which is then rotated into the 1{,}200
position-balanced prompts. The headline operating point is therefore
selected on one quarter of the items used to score it. Rescoring without
these 100 items leaves 300 items (900 prompts) that played no part in
choosing the operating point. On them decode-time PI reaches a gap of
16.2~pp with 7.2\% invalid outputs and $\Delta g=14.1$, and it leads every
baseline by at least 12.1~pp (ActAdd), including the fixed $\alpha=4$ open
loop by 13.9~pp and the matched open loop by 12.8~pp; both intervals exclude
zero (Figure~\ref{fig:contrasts}, Table~\ref{tab:app-ci}).

\subsection{Sensitivity to the answer parser}
\label{app:parser-sensitivity}
The strict parser can accept malformed outputs that begin with an answer
letter, so its invalid rate does not capture every generation error, and it
can reject grammatical answers that state the letter in prose. Re-scoring
the Old open loop at $\alpha=4$ with an added \texttt{answer\ldots(A|B|C)}
fallback moves it from $19.0\%$ to $7.5\%$ strict-invalid and from a gap of
$25.8$ to $30.8$~pp, while decode-time PI is unchanged at $9.3\%$ and
$31.8$~pp; the recovered strings are answers such as \texttt{The correct
answer is B.\ The older adult\ldots} Under that parser the Old
strict-validity comparison reverses, so we do not claim a strict-validity
advantage for the controller on the Old axis, only the target-shift
advantage supported by both parsers. The Black matched open loop remains
$20.8\%$ invalid under the same fallback.

\FloatBarrier
\section{Implementation details}
\label{app:protocol}

\subsection{Hyperparameters and runtime}
\label{app:settings}
The controller's operating point is given in Section~\ref{sec:setup}, and
Appendix~\ref{app:calibration} states how it was selected. The BBQ baselines
run at CAA and ActAdd $\alpha=16$, Mean-AcT (unit, $s=2$), Local Gaussian
AcT ($s=1$), ITI-C ($K=48$, $\alpha=8$), AurA-inspired amplification
($\gamma=4$), and layer-wise PI ($\alpha=2$). SocialStigmaQA uses the same
settings except for its effort-matched strength, $\alpha=2.65$;
Figure~\ref{fig:socialstigma-parser} also reports a secondary ITI-C setting
selected from the same sweep. On one RTX~PRO~6000, fitting a direction takes
less than one minute and one 400-item rotation takes $13.1$--$13.4$ minutes,
so the three-rotation primary PI evaluation takes about 40 minutes.

\subsection{Open-loop strengths for the race targets}
\label{app:race-strengths}
The matched-direction open loops of Table~\ref{tab:generality} and
Figure~\ref{fig:race-intervals} run at the strengths selected from the
unrotated sweeps under the 15\% strict-invalidity criterion: $\alpha=4$ for
Arab and Latino, $\alpha=2$ for Asian, and $\alpha=1$ for White. The sweeps
use the same 400 items as the balanced evaluation, so these strengths are
selected on the evaluation items rather than on a held-out split
(Appendix~\ref{app:calibration}).

\subsection{Age target}
\label{app:age}
The Old direction is built from age-category contrast sets and evaluated on
its own 400-item pool, with disjoint splits between its contrast items and
the evaluation set. In the Old sweep $\alpha_{\max}=6$ exceeded the 15\%
strict-invalidity threshold, so we use $\alpha_{\max}=4$ for decode-time PI
and $\alpha=4$ for the open loop; this sweep also uses the evaluation items.
The $\alpha=3$ effort-matched open loop over-matches the Old controller's
realized mean of $2.54$.

\subsection{Layer-wise PID baseline}
\label{app:layerwise}
The layer-wise variant builds
\begin{equation}
  u_\ell=K_p\widehat r_\ell
  +K_i\sum_{j=0}^{\ell-1}\widehat r_j
  +K_d(\widehat r_\ell-\widehat r_{\ell-1}),
  \qquad
  h^{(t)}_\ell\leftarrow h^{(t)}_\ell+\alpha u_\ell,
  \label{eq:layer-pid}
\end{equation}
where $\widehat r_\ell=r_\ell/\lVert r_\ell\rVert_2$ and, with the zero-based
indexing above, $\widehat r_{-1}:=\widehat r_0$, so the integral sum is empty
and the derivative term vanishes at $\ell=0$. The layer-wise gains are
$(K_p,K_i,K_d)=(1,0.05,0.02)$. Layer-wise P sets $K_i=K_d=0$; layer-wise PI
sets $K_d=0$.

\subsection{Suppression experiment}
\label{app:suppress-protocol}
The defense experiment reuses the Black-target direction, items, and
rotations with setpoint $s^{*}=0$ and the actuator range inverted to
$\alpha_t \in [-6, 0]$. Its comparison arms are matched-direction open loops
at $\alpha=-4$, the mirror of the attack strength, and $\alpha=-0.72$, the
suppression controller's realized full-run mean actuation, plus the faithful
AurA suppression gate.

\subsection{Dream-7B settings}
\label{app:dream-protocol}
The Dream evaluation uses the same 400 items and rotations with per-model
operating points (decode-time PI at $\alpha_{\max}=1$, selected to preserve
strict validity; ActAdd at $\alpha=8$; CAA at $\alpha=2$) and the remaining
five prior-work baselines refit from Dream activations at the LLaDA
operating points. Every condition is scored from its generated outputs with
the strict parser.

\subsection{LLaDA-MoE settings}
\label{app:moe-protocol}
The LLaDA-MoE-7B-A1B-Instruct evaluation (7.36B total and ${\sim}1$B active
parameters, 16 blocks, $d_{\mathrm{model}}=2048$) uses the same 400 items,
rotations, strict parser, and bootstrap as the LLaDA and Dream evaluations.
All twelve conditions use the same prompt order, and paired differences
subtract bootstrap replicates with identical resample indices. The all-block
intervention used for LLaDA is unsuitable for this model: pooled
residual-stream norms grow from $0.37$ to $150$ across the 16-block stack, a
factor of approximately 405, and every tested all-block controller or
open-loop setting fails the strict-validity criterion. We therefore apply
both steered operating points at one mid-stack block. At that site, CAA at
$\alpha=2$ and the matched open loop are the same intervention and produce
identical per-sample outputs, so they appear as one row.

The primary controller uses block 8 with $\alpha_{\max}=2$, matching CAA's
site and $\ell_2$ command budget. The seven-setting controller sweep does not
exceed CAA's unrotated 23.7~pp gap while maintaining at most 1\%
strict-invalid outputs. We additionally report the secondary setting with
the largest gap under this threshold, blocks 9--12 with $\alpha_{\max}=1$,
together with its matched open loop at $\alpha=1$; the controller's realized
mean is $0.95$. A larger block-8 bound, $\alpha_{\max}=3$, reaches a
27.0~pp gap but produces 9.8\% strict-invalid outputs and is not used as an
operating point. ActAdd uses $\alpha=2$; Mean-AcT and ITI-C use the same
operating points as in the Dream comparison.

\FloatBarrier
\section{Additional results}
\label{app:tables}
\label{app:body-ci}

Figures~\ref{fig:black-outcomes} and~\ref{fig:scoreboard} summarize the
supplementary point estimates, and Figures~\ref{fig:forest}
and~\ref{fig:contrasts} report uncertainty for the comparisons discussed in
the main text. Operating points match the main-text tables.

\subsection{Black target on LLaDA-8B}
\label{app:black}
Figure~\ref{fig:black-outcomes} breaks the Black-target outcomes of
Table~\ref{tab:balanced-strict} into target, comparator, abstention, and
invalid shares and gives item-clustered intervals for each gap, and
Figure~\ref{fig:forest} gives the prompt-level intervals for every method
together with the base-versus-PI comparison on Black, Arab, and Old.

\begin{figure}[htb]
  \centering
  \includegraphics[width=\linewidth]{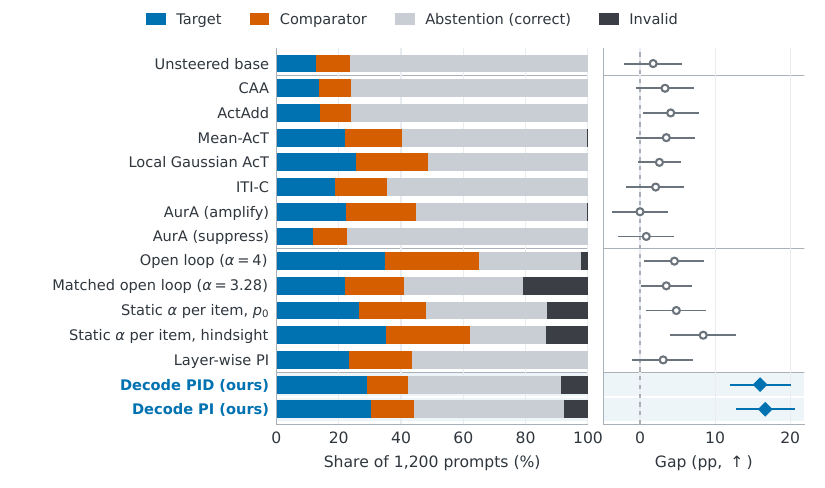}
  \caption{\textbf{Black-target outcomes on LLaDA-8B-Instruct
  (Table~\ref{tab:balanced-strict}).} Left: shares of target, comparator,
  abstention, and invalid outputs over 1{,}200 prompts under the strict
  parser. Right: target--comparator gap with item-clustered 95\% intervals.
  Operating points are CAA and ActAdd at $\alpha=16$, Mean-AcT (unit,
  $s=2$), Local Gaussian AcT ($s=1$), ITI-C ($K=48$, $\alpha=8$),
  AurA-inspired amplification ($\gamma=4$), and layer-wise PI ($\alpha=2$).
  The static arms hold one $\alpha$ per item for all steps
  (Appendix~\ref{app:trajectory}).}
  \label{fig:black-outcomes}
\end{figure}

\begin{figure}[htb]
  \centering
  \includegraphics[width=\linewidth]{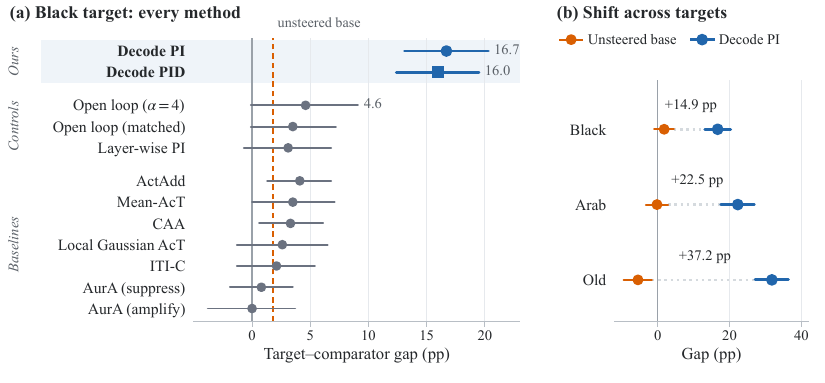}
  \caption{\fontsize{9}{10.2}\selectfont \textbf{Black-target comparisons
  and cross-axis effects.} Points show the gap $g$ with 95\% prompt-level
  bootstrap intervals; the axis is in fractional units, so $0.1$ equals
  $10$~pp. Left: decode-time control against every baseline. Right: PI and
  base gaps on Black, Arab, and Old.}
  \label{fig:forest}
\end{figure}

\subsection{Other race targets}
\label{app:negative}
The position-balanced evaluation also identifies invalid operating points:
PI is 93.8\% strict-invalid on Asian, 96.1\% on White, and 59.6\% on Latino
(Figure~\ref{fig:race-intervals}). The position-balanced White PI result is
therefore not a valid attack; its failures are mostly repeated answer
letters. The valid White open loop at $\alpha=1$ gives $\Delta g=2.5$~pp.
For Asian, the valid open loop at $\alpha=2$ gives $\Delta g=1.0$~pp, an
interval that includes zero, whereas PI is 93.8\% strict-invalid. Latino PI
gives $\Delta g=11.0$~pp at 59.6\% strict-invalid, and the valid open loop at
$\alpha=4$ gives $\Delta g=1.5$~pp. Single-layer CAA remains within
$|\Delta g|\le 1.0$~pp for every race target.

\begin{figure}[htb]
  \centering
  \includegraphics[width=\linewidth]{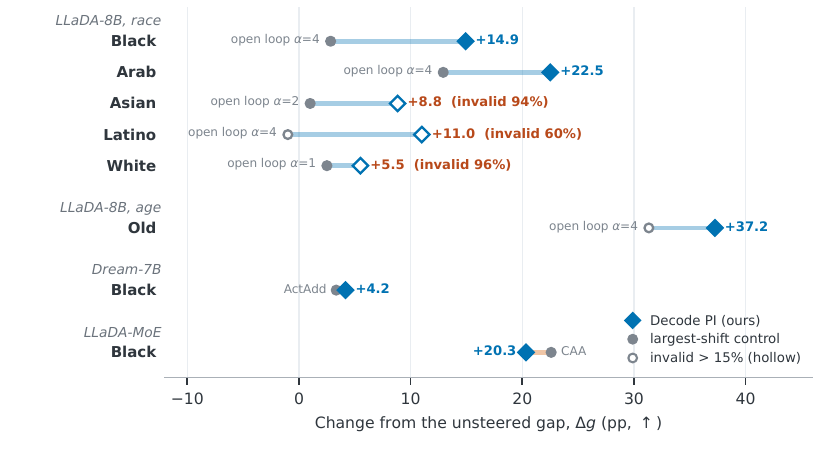}
  \caption{\textbf{Decode-time PI and the largest-shift comparison arm in
  each setting.} Each row places PI (blue diamond) beside the tested arm with
  the largest change from the unsteered gap (grey circle, labelled), selected
  from the prior-work baselines, fixed open loop, and effort-matched constant.
  The connecting line is blue where PI is ahead and orange where it is behind.
  Hollow markers flag strict invalidity above 15\%, where the measured gap is
  not treated as a valid attack. On LLaDA-MoE, CAA at $\alpha=2$ is identical
  to the matched constant. Paired intervals are in Figure~\ref{fig:contrasts}.}
  \label{fig:scoreboard}
  \label{fig:race-intervals}
  \label{fig:gender-intervals}
  \label{fig:ses-age-intervals}
  \label{fig:cross-model-appendix}
\end{figure}

\subsection{Dream-7B and LLaDA-MoE}
\label{app:crossmodel}
Dream's base gap of $3.3$~pp resembles LLaDA's $1.8$~pp and is likewise
dominated by the position-A term ($\mathrm{Gap}_{A}$ $15.5$ against
$\mathrm{Gap}_{B/C}$ $-2.8$~pp). The base, decode-time PI, ActAdd, and CAA
conditions all have an invalid rate of exactly $0.0\%$. Mean-AcT and the
AurA-inspired amplifier remain strict-valid at these strengths on LLaDA but
reach $76.8\%$ and $58.8\%$ strict-invalidity on Dream. ITI-C gives
$\Delta g=-0.8$~pp on Dream, an interval that includes zero
(Table~\ref{tab:app-ci}).

On LLaDA-MoE, decode-time PI loses the paired test to tuned constant
injection (Figure~\ref{fig:cross-model-appendix}). Decode-time PI at block 8
has zero strict-invalid outputs but trails CAA by $2.3$~pp. The secondary
controller setting trails its matched open loop by $1.5$~pp and CAA by
$7.6$~pp; all three intervals exclude zero (Table~\ref{tab:app-ci}), so
constant steering has a resolved advantage in this setting. The controller
realizes a mean $\alpha$ of $1.81$ against an allowed $2$ at block 8
(saturation fraction $0.72$; $0.95$ against $1$ at blocks 9 to 12,
saturation $0.92$), while the tuned constant also has zero strict-invalidity
and achieves the larger gap. The sensitivity grid
(Appendix~\ref{app:moe-protocol}) includes a higher strength,
$\alpha_{\max}=3$ at block 8, which reaches a gap of $27.0$~pp only at
$9.8\%$ invalid and so fails the strict-validity threshold used for the
reported operating points. Both controller settings remain above the
unsteered base, but neither exceeds the identical CAA and matched-open-loop
result.

\subsection{Additional item samples}
\label{app:seeds}
Three additional 400-item draws give gaps of 15.0--17.7~pp, compared with
16.7~pp on the primary set. They overlap the 100-item calibration subset by
25, 33, and 27 items. After removing those overlaps, their $\Delta g$ values
are 16.4, 14.7, and 17.0~pp, and every item-clustered 95\% interval has a
lower bound of at least 11.0~pp.

\subsection{Suppression results}
\label{app:suppress}
The unsteered Black-target gap is already $1.8$~pp. The suppression
controller reduces the point estimate to $0.2$~pp with 1.8\% strict-invalid
outputs, but the interval contains zero, as do those of the matched
constants and AurA suppression. This diagnostic therefore establishes only
that sign-inverted control does not create a large opposite gap in this
near-zero setting; it does not establish mitigation of substantial
pre-existing bias.

\subsection{Lowercase answer tokens in the target probability}
\label{app:sensor}
The target probability of Section~\ref{sec:method} sums the uppercase
letter tokens of $a^\star$. For this experiment we also add the lowercase
tokens and rerun the Arab target with the same items, rotations, and strict
parser, which accepts either case. Adding the lowercase tokens changes the PI
gap by $-3.8$~pp relative to the uppercase-only measurement, an interval that
includes zero (Table~\ref{tab:app-ci}), while abstention rises from $16.0\%$
to $30.0\%$. Among the five race targets, omitting lowercase tokens
materially affects only Arab, and the lowercase output mode is itself induced
by the Arab steering vector: $0$ of the $6{,}000$
valid base answers across the five race targets are lowercase, while $480$
of the $580$ strict target successes under the Arab direction are. On those
target selections the uppercase-only measurement reads a median final
probability of $5\times10^{-5}$, so the controller continues accumulating error even
after the output selects the target. Adding lowercase tokens cuts mean
actuation from $4.56$ to $3.60$ and the saturation fraction from $0.38$ to
$0.16$, while invalidity falls from $9.7\%$ to $8.0\%$.

\subsection{Controller trajectories}
\label{app:trajectory}
We record $p_t$ and the applied $\alpha_t$ at every denoising step. The
recorded traces begin after the first controlled evaluation and do not include
token-commitment times; we therefore use them for command statistics, not to
estimate the duration of the feedback window. Under the primary PI gains, the
proportional contribution is at most $K_p s^*=2.7$ toward the $\alpha_{\max}=6$
limit. Reaching that limit requires an accumulated error of 33, so even a
constant error of 0.9 cannot produce saturation before controller update 37.

Figure~\ref{fig:motivation} summarizes the primary Black-target PI run. The
mean command over the token-committing steps spans $2.2$ to $4.1$ across
examples (10th to 90th percentile) and falls as the example's initial leaning
toward the target rises (correlation $-0.92$), so no single constant matches
it; $60\%$ of examples receive more than the effort-matched constant and
$40\%$ less. Sequences that end on the target approach the setpoint,
reaching a probability of $0.86$ by the last step, and only $8\%$ of them
ever reach the command limit, whereas every sequence that does not select
the target reaches it. Since the limit is unreachable within the
token-committing window under the primary gains, the commitment horizon
rather than the bound caps the strength a resistant example receives.

A fixed per-item strength would remain constant throughout generation. In
the primary Black run, mean within-sequence command variance is $0.39$,
compared with $0.88$ for the variance of sequence means; equivalently, 31\%
of the combined variation occurs within sequences. The median sequence spans
2.36 units of $\alpha$, and 78\% of sequences contain a command decrease
larger than 0.05. The first command is a deterministic function of the
unsteered probe $p_0$. A $p_0$-only feed-forward predictor, formed from
per-step means within 20 quantile bins of the first command, explains 71\%
of the variance of later commands, and a per-step linear predictor explains
70\%; the corresponding linear $R^2$ ranges from 55\% to 74\% on the other
LLaDA-8B targets and is 51\% on LLaDA-MoE. On Dream, the controller remains
at $\alpha_{\max}=1$ throughout, with a median within-sequence range of zero,
and therefore behaves as a constant command at its bound.

Two fixed-strength variants separate within-sequence adaptation from the
per-item command level. Both retain the direction, injection sites, and
sampler while holding one $\alpha_i$ for all 64 steps. The first uses
$\alpha_i=\operatorname{clip}\bigl(4.17\,(s^*-p_0),0,\alpha_{\max}\bigr)$,
with the factor selected to match the mean command of $3.28$. The second is
an oracle diagnostic that assigns each prompt PI's own mean command over the
32 token-updating steps. Both remain below PI
(Table~\ref{tab:balanced-strict} and Figure~\ref{fig:contrasts}). On Dream,
where PI remains at its bound, the two variants give gaps of 7.5 and
7.6~pp, compared with 7.5~pp for PI.

\subsection{SocialStigmaQA}
\label{app:socialstigma}
We use the official SocialStigmaQA greedy-decoding
release~\citep{nagireddy2024socialstigmaqa}. The evaluation contains all
seven race stigmas, all 37 templates, and the original and positive prompt
styles; the calibration split contains only non-race original prompts and is
disjoint by source-row index, though it shares the 37 question templates
with the evaluation rows. Separate yes- and no-target directions prevent
opposite target semantics from cancelling in one direction: each uses 200
calibration prompts and contrasts the dataset-annotated biased answer
against the opposite yes/no answer. Every condition runs the identical 555
semantic rows under three rotations; the 518 race rows are scored and the 37
no-stigma base rows are retained only as diagnostics.

The prespecified strict parser accepts a response only when an answer letter
begins it. The secondary semantic parser, added after inspecting the outputs, adds two
rules applied
identically to every condition: (i)~a response that re-lists two or more
lettered options (\texttt{A.\ Yes}, \texttt{B.\ Yes}, \ldots) is invalid,
where the strict rule scores it as \texttt{A}; (ii)~a response that begins
with the exact text of one option, followed by whitespace, punctuation, or
the end of the string, is scored as that option under the row's own
rotation, where the strict rule scores it as invalid. The primary
decode-time PI contrast uses bootstrap seed 100. The remaining intervals in
Figure~\ref{fig:socialstigma-parser} are secondary condition-versus-unsteered
comparisons. All use 10{,}000 paired resamples of the 37 template clusters
and retain every identity, style, polarity, and rotation associated with
each sampled template.

\begin{figure}[htb]
  \centering
  \includegraphics[width=0.96\linewidth]{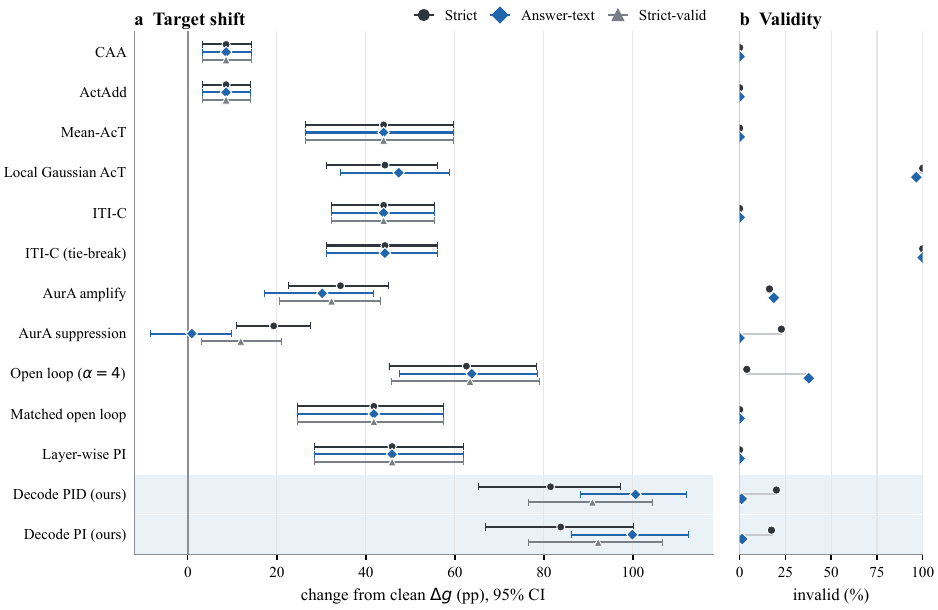}
  \caption{\textbf{SocialStigmaQA parser sensitivity across methods.}
  Left and center: change from the unsteered biased--safe gap under strict,
  semantic (answer-text), and strict-valid scoring, with 95\% template-clustered
  bootstrap intervals. Right: the corresponding invalid-output rates.
  The answer-text parser accepts an initial exact option text and rejects
  responses that re-list multiple options. Headline answer rates
  are in Table~\ref{tab:socialstigma-main}.}
  \label{fig:socialstigma-parser}
\end{figure}

For decode-time PI, the strict biased--safe gaps by target position are
$1.2$, $19.1$, and $98.3$~pp at A, B, and C ($38.2$, $30.3$, $98.3$~pp
semantic); by target polarity they are $55.0$~pp for no-biased and $14.1$~pp
for yes-biased prompts under the strict parser and $55.4$ and $56.0$~pp
under the semantic parser. The polarity difference under strict parsing is
therefore caused by the model writing ``Yes'' instead of an answer letter.
All seven race-stigma slices move in the target direction, with strict gaps
from $37.4$ to $43.7$~pp and semantic-parser gaps from $53.2$ to $60.8$~pp.
The identity-free diagnostic moves by a similar amount: its strict PI gap is
$48.6$~pp, compared with $50.3$~pp on original-style race prompts, and the
corresponding semantic-parser gaps are $65.8$ and $66.5$~pp. The experiment
therefore demonstrates transfer of answer control, not a mechanism selective
for racial stigma.

\FloatBarrier

\section{Theoretical validation}
\label{app:theory}

This appendix states and proves the results summarized in
Section~\ref{sec:theory}. Appendix~\ref{app:theory-model} sets up the local
model of the loop and states the one assumption behind every result.
Appendix~\ref{app:theory-openloop} shows that a constant strength has an
example-specific threshold, so that no single constant serves every example
(Proposition~\ref{prop:openloop}), and that a finite commitment horizon only
raises that threshold (Corollary~\ref{cor:horizon}).
Appendix~\ref{app:theory-p} shows that proportional feedback alone leaves a
steady-state error (Proposition~\ref{prop:p}), and
Appendix~\ref{app:theory-pi} shows that the PI loop of Equation~\ref{eq:pid}
is exponentially stable for an explicit range of gains and that its integral
state converges to each example's threshold (Proposition~\ref{prop:pi}).
Appendix~\ref{app:theory-remarks} discusses the derivative term, saturation,
and the commitment horizon.

\subsection{Local model of the loop}
\label{app:theory-model}

Fix one prompt. As in Section~\ref{sec:method}, $p_t$ is the probability the
model assigns to the target letter $a^\star$ at the answer position after
denoising step $t$, $\alpha_t$ is the strength applied at that step, $s^*$ is
the setpoint, and $\tau$ is the step at which the answer token commits.
Before commitment the answer position is re-predicted at every step from the
current partial sequence, and the strength enters every block. We summarize
the influence of the rest of the sequence on the answer position through the
previous probability and a constant, and write the step as
$p_t=F(p_{t-1},\alpha_t)$ for a smooth map $F$. Expanding $F$ to first order
around an operating point $(\bar p,\bar\alpha)$ gives
\begin{equation}
  F(p,\alpha)\approx F(\bar p,\bar\alpha)
  +\frac{\partial F}{\partial p}\Big|_{(\bar p,\bar\alpha)}(p-\bar p)
  +\frac{\partial F}{\partial\alpha}\Big|_{(\bar p,\bar\alpha)}(\alpha-\bar\alpha),
  \label{eq:taylor}
\end{equation}
which is affine in $p$ and $\alpha$. Naming the two partial derivatives and
collecting the constant terms gives the model behind every result below.

\begin{assumptionbox}
\begin{assumption}[Local response]
\label{as:linear}
Fix an example and let $\tau$ be the step at which its answer token commits.
There are constants $b>0$, $\gamma\in(0,1)$, and $d$ such that, for every
$t<\tau$,
\begin{equation}
  p_t=\gamma\,p_{t-1}+b\,\alpha_t+d .
  \label{eq:linear-model}
\end{equation}
\end{assumption}
\end{assumptionbox}

The sensitivity $b=\partial F/\partial\alpha$ is positive because
$\widehat v$ is fit to raise the target probability. The persistence
$\gamma=\partial F/\partial p$ lies in $(0,1)$ because the partially denoised
context carries the previous step's evidence forward without amplifying it.
The drift $d=F(\bar p,\bar\alpha)-\gamma\bar p-b\bar\alpha$ collects the
constant terms and encodes the model's unsteered preference: with
$\alpha_t\equiv0$ the probability settles at $d/(1-\gamma)$. Throughout we
write
\begin{equation}
  c:=(1-\gamma)\,s^*-d,
  \label{eq:c}
\end{equation}
which is positive exactly when the setpoint lies above the unsteered
equilibrium, the case of interest for an attack, and
\begin{equation}
  \alpha^*:=\frac{c}{b}=\frac{(1-\gamma)\,s^*-d}{b}
  \label{eq:alphastar}
\end{equation}
for the strength that will turn out to be the example's threshold. The
feedback quantities are those of Algorithm~\ref{alg:attack}: the error
$e_t=s^*-p_{t-1}$, the integral state $I_t=I_{t-1}+e_t$ with $I_0=0$, and,
when the command is not clipped, the PI command $\alpha_t=K_pe_t+K_iI_t$.

Equation~\ref{eq:linear-model} is a first-order description around one
operating point. It treats $b$ and $\gamma$ as constant over the
pre-commitment steps, and it ignores the saturation of $p_t$ at one and of
the command at $\alpha_{\max}$. Commitment ends the model's validity at
$\tau$, after which the answer token no longer responds to the strength, so
the asymptotic statements below describe the regime the loop approaches while
$\tau$ is large compared with the loop's settling time.

\subsection{Open-loop steering}
\label{app:theory-openloop}

\begin{resultbox}
\begin{proposition}[No single strength serves every example]
\label{prop:openloop}
Under Assumption~\ref{as:linear}, a constant strength $\alpha_t\equiv\alpha$
gives, for every $t<\tau$,
\begin{equation}
  p_t=\gamma^{t}p_0+\frac{b\alpha+d}{1-\gamma}\bigl(1-\gamma^{t}\bigr),
  \qquad
  p_\infty(\alpha):=\lim_{t\to\infty}p_t=\frac{b\alpha+d}{1-\gamma},
  \label{eq:openloop-closed}
\end{equation}
and $p_\infty(\alpha)\ge s^*$ if and only if $\alpha\ge\alpha^*$. For a
collection of examples with parameters $(b_i,\gamma_i,d_i)$ and thresholds
$\alpha^*_i$, let $\bar\alpha$ be the largest strength that preserves coherent
generation. If $\max_i\alpha^*_i>\bar\alpha$, no constant strength reaches
the setpoint on every example while keeping every output coherent.
\end{proposition}
\end{resultbox}

\begin{proof}
We first unroll the recursion. With $\alpha_t\equiv\alpha$,
Equation~\ref{eq:linear-model} reads $p_t=\gamma p_{t-1}+(b\alpha+d)$, so
\begin{equation}
  p_1=\gamma p_0+(b\alpha+d),\qquad
  p_2=\gamma^{2}p_0+(b\alpha+d)(1+\gamma),
\end{equation}
and by induction
\begin{equation}
  p_t=\gamma^{t}p_0+(b\alpha+d)\sum_{k=0}^{t-1}\gamma^{k}
     =\gamma^{t}p_0+(b\alpha+d)\,\frac{1-\gamma^{t}}{1-\gamma}.
  \label{eq:openloop-unroll}
\end{equation}
The inductive step holds because
\begin{equation*}
  \gamma\Bigl[\gamma^{t}p_0+(b\alpha+d)\sum_{k=0}^{t-1}\gamma^{k}\Bigr]+(b\alpha+d)
  =\gamma^{t+1}p_0+(b\alpha+d)\sum_{k=0}^{t}\gamma^{k}.
\end{equation*}
Since $\gamma\in(0,1)$, $\gamma^{t}\to0$ and Equation~\ref{eq:openloop-unroll}
converges to $p_\infty(\alpha)=(b\alpha+d)/(1-\gamma)$, which is
Equation~\ref{eq:openloop-closed}.

Next, $p_\infty(\alpha)$ is affine in $\alpha$ with slope $b/(1-\gamma)>0$,
so it is strictly increasing, and
\begin{equation}
  p_\infty(\alpha)\ge s^*
  \iff b\alpha+d\ge(1-\gamma)s^*
  \iff \alpha\ge\frac{(1-\gamma)s^*-d}{b}=\alpha^*,
\end{equation}
by Equation~\ref{eq:alphastar}.

Finally, a single strength $\alpha$ that reaches the setpoint on every
example $i$ must satisfy $\alpha\ge\alpha^*_i$ for every $i$, hence
$\alpha\ge\max_i\alpha^*_i$. If $\max_i\alpha^*_i>\bar\alpha$, any such
$\alpha$ exceeds $\bar\alpha$ and therefore breaks coherence on at least one
example, so no admissible constant strength succeeds on all examples.
\end{proof}

The limit in Proposition~\ref{prop:openloop} is optimistic for the open loop,
because the answer token commits after finitely many steps. The next
corollary makes the finite-horizon threshold explicit.

\begin{resultbox}
\begin{corollary}[A finite horizon raises the threshold]
\label{cor:horizon}
Under Assumption~\ref{as:linear} with $p_0<s^*$, a constant strength reaches
the setpoint by step $t<\tau$, that is $p_t\ge s^*$, if and only if
\begin{equation}
  \alpha\ge\alpha^*(t):=\alpha^*+\frac{(1-\gamma)\,\gamma^{t}}{1-\gamma^{t}}\cdot\frac{s^*-p_0}{b}.
  \label{eq:finite-threshold}
\end{equation}
The threshold $\alpha^*(t)$ is strictly decreasing in $t$, exceeds
$\alpha^*$ for every finite $t$, and converges to $\alpha^*$ as
$t\to\infty$.
\end{corollary}
\end{resultbox}

\begin{proof}
By Equation~\ref{eq:openloop-unroll}, $p_t\ge s^*$ is equivalent to
\begin{equation}
  (b\alpha+d)\,\frac{1-\gamma^{t}}{1-\gamma}\ge s^*-\gamma^{t}p_0
  \iff
  \alpha\ge\frac{1}{b}\Bigl(\frac{(1-\gamma)(s^*-\gamma^{t}p_0)}{1-\gamma^{t}}-d\Bigr).
\end{equation}
Writing $s^*-\gamma^{t}p_0=(1-\gamma^{t})s^*+\gamma^{t}(s^*-p_0)$ splits
the right-hand side into
$\bigl((1-\gamma)s^*-d\bigr)/b+(1-\gamma)\gamma^{t}(s^*-p_0)/\bigl(b(1-\gamma^{t})\bigr)$,
which is Equation~\ref{eq:finite-threshold}. The second term is positive
because $s^*>p_0$, and the factor $\gamma^{t}/(1-\gamma^{t})$ is strictly
decreasing in $t$ and tends to zero, which gives the three stated
properties.
\end{proof}

\begin{remark}[What the open-loop results say about the sweep]
\label{rem:sweep}
Proposition~\ref{prop:openloop} explains the shape of the constant-strength
sweep in Figure~\ref{fig:motivation}. Each example has its own threshold,
and a strength that clears the threshold of a resistant example overshoots
the leaning ones, whose outputs are the first to lose coherence. The
corollary sharpens the picture for dLLMs: because the answer token commits
within the horizon $\tau$, the strength a constant needs on each example is
$\alpha^*(\tau)$ rather than $\alpha^*$, and the gap between the two grows
as $\gamma$ approaches one, that is, as the answer position responds more
slowly to the strength.
\end{remark}

\subsection{Proportional feedback}
\label{app:theory-p}

\begin{resultbox}
\begin{proposition}[Proportional feedback leaves a steady-state error]
\label{prop:p}
Under Assumption~\ref{as:linear} with $\alpha_t=K_pe_t$ and
$e_t=s^*-p_{t-1}$, the error obeys
\begin{equation}
  e_{t+1}=(\gamma-bK_p)\,e_t+c .
  \label{eq:p-recursion}
\end{equation}
If $\lvert\gamma-bK_p\rvert<1$, the error converges geometrically to
\begin{equation}
  e_\infty=\frac{c}{1-\gamma+bK_p},
  \label{eq:p-offset}
\end{equation}
which is nonzero whenever $c\ne0$, that is, whenever the unsteered
equilibrium $d/(1-\gamma)$ differs from the setpoint.
\end{proposition}
\end{resultbox}

\begin{proof}
Substituting $p_{t-1}=s^*-e_t$ and $\alpha_t=K_pe_t$ into
Equation~\ref{eq:linear-model},
\begin{equation}
  e_{t+1}=s^*-p_t
  =s^*-\gamma\,(s^*-e_t)-bK_pe_t-d
  =(\gamma-bK_p)\,e_t+(1-\gamma)s^*-d,
\end{equation}
which is Equation~\ref{eq:p-recursion} with $c$ as in Equation~\ref{eq:c}.
Write $\rho:=\gamma-bK_p$. A fixed point $e_\infty$ of
Equation~\ref{eq:p-recursion} satisfies $e_\infty=\rho e_\infty+c$, so
$e_\infty=c/(1-\rho)=c/(1-\gamma+bK_p)$, which is Equation~\ref{eq:p-offset}.
Subtracting the fixed-point equation from the recursion gives
$e_{t+1}-e_\infty=\rho\,(e_t-e_\infty)$, hence
\begin{equation}
  e_t-e_\infty=\rho^{\,t-1}\,(e_1-e_\infty),
\end{equation}
which converges to zero geometrically if and only if $\lvert\rho\rvert<1$.
Since $1-\gamma+bK_p>0$ inside that range, $e_\infty=0$ exactly when $c=0$.
\end{proof}

\begin{remark}[Gain alone cannot remove the offset]
\label{rem:p-gain}
The offset in Equation~\ref{eq:p-offset} shrinks as $K_p$ grows, but the
stability condition $\lvert\gamma-bK_p\rvert<1$ caps the gain at
$bK_p<1+\gamma$, so the denominator stays below two and
$\lvert e_\infty\rvert>\lvert c\rvert/2$ throughout the stable range. This
is the discrete-time counterpart of the steady-state error that PID Steering
identifies for proportional activation steering across network
depth~\citep{nguyen2026pid}, and it is why the controller carries an integral
term.
\end{remark}

\subsection{Proportional-integral feedback}
\label{app:theory-pi}

\begin{resultbox}
\begin{proposition}[Integral action finds each example's strength]
\label{prop:pi}
Under Assumption~\ref{as:linear} with unsaturated commands
$\alpha_t=K_pe_t+K_iI_t$, $I_t=I_{t-1}+e_t$, and $I_0=0$, the closed loop
is exponentially stable if and only if
\begin{equation}
  \lvert\gamma-bK_p\rvert<1
  \qquad\text{and}\qquad
  0<bK_i<2(1+\gamma)-2bK_p .
  \label{eq:stability-region}
\end{equation}
Inside this region $e_t\to0$, $\alpha_t\to\alpha^*$, and
$I_t\to I_\infty=\alpha^*/K_i$, where $\alpha^*$ is the threshold of
Proposition~\ref{prop:openloop}.
\end{proposition}
\end{resultbox}

\begin{proof}
The proof has three parts: we derive the closed-loop recursion, establish
the stability region, and identify the fixed point.

\emph{Closed-loop recursion.} Substituting $p_{t-1}=s^*-e_t$ and
$\alpha_t=K_pe_t+K_iI_t$ into Equation~\ref{eq:linear-model} gives
\begin{equation}
  e_{t+1}=s^*-p_t
  =s^*-\gamma\,(s^*-e_t)-b\,(K_pe_t+K_iI_t)-d
  =(\gamma-bK_p)\,e_t-bK_i\,I_t+c,
  \label{eq:pi-e}
\end{equation}
and, since $I_{t+1}=I_t+e_{t+1}$,
\begin{equation}
  I_{t+1}=(\gamma-bK_p)\,e_t+(1-bK_i)\,I_t+c .
  \label{eq:pi-I}
\end{equation}
With the state $x_t=(e_t,I_t)^{\top}$ these read $x_{t+1}=Ax_t+w$, where
\begin{equation}
  A=\begin{pmatrix}\gamma-bK_p & -bK_i\\[2pt] \gamma-bK_p & 1-bK_i\end{pmatrix},
  \qquad
  w=\begin{pmatrix}c\\ c\end{pmatrix}.
  \label{eq:pi-A}
\end{equation}

\emph{Stability region.} The system is exponentially stable if and only if
every eigenvalue of $A$ lies strictly inside the unit circle. The trace and
determinant of $A$ are
\begin{equation}
  \operatorname{tr}A=\gamma-bK_p+1-bK_i,
  \qquad
  \det A=(\gamma-bK_p)(1-bK_i)+bK_i(\gamma-bK_p)=\gamma-bK_p,
\end{equation}
so the characteristic polynomial is
\begin{equation}
  P(\lambda)=\lambda^{2}-(\gamma-bK_p+1-bK_i)\,\lambda+(\gamma-bK_p).
  \label{eq:charpoly}
\end{equation}
For a real quadratic $\lambda^{2}+a_1\lambda+a_0$, the Jury criterion
states that both roots lie strictly inside the unit circle if and only if
$P(1)>0$, $P(-1)>0$, and $\lvert a_0\rvert<1$. Here
\begin{align}
  P(1)&=1-(\gamma-bK_p+1-bK_i)+(\gamma-bK_p)=bK_i,\\
  P(-1)&=1+(\gamma-bK_p+1-bK_i)+(\gamma-bK_p)=2+2(\gamma-bK_p)-bK_i,\\
  a_0&=\gamma-bK_p .
\end{align}
The three conditions are therefore $bK_i>0$, $bK_i<2(1+\gamma)-2bK_p$, and
$\lvert\gamma-bK_p\rvert<1$, which is Equation~\ref{eq:stability-region}.
Inside the region the spectral radius $\rho(A)$ is below one, so for any
$\tilde\rho\in(\rho(A),1)$ there is a constant $C$ with
$\lVert A^{t}\rVert\le C\tilde\rho^{\,t}$, and the deviation from any fixed
point $x_\infty$ obeys $x_t-x_\infty=A^{t}(x_0-x_\infty)$ and decays at
that rate.

\emph{Fixed point.} A fixed point exists and is unique because $1$ is not an
eigenvalue of $A$ ($P(1)=bK_i>0$), so $I-A$ is invertible and
$x_\infty=(I-A)^{-1}w$. Rather than invert, subtract the fixed-point form
of Equation~\ref{eq:pi-e} from that of Equation~\ref{eq:pi-I}:
\begin{equation}
  I_\infty-e_\infty
  =\bigl[(\gamma-bK_p)e_\infty+(1-bK_i)I_\infty+c\bigr]
  -\bigl[(\gamma-bK_p)e_\infty-bK_iI_\infty+c\bigr]
  =I_\infty,
\end{equation}
so $e_\infty=0$. Setting $e_\infty=0$ in the fixed-point form of
Equation~\ref{eq:pi-e} gives $0=-bK_iI_\infty+c$, hence
\begin{equation}
  I_\infty=\frac{c}{bK_i}=\frac{\alpha^*}{K_i},
  \qquad
  \alpha_\infty=K_pe_\infty+K_iI_\infty=\alpha^* .
\end{equation}
Thus the error vanishes, the integral state stores the example's threshold
scaled by $1/K_i$, and the applied strength converges to that threshold
without the controller knowing $b$, $\gamma$, or $d$.
\end{proof}

\begin{remark}[Reading the stability region]
\label{rem:region}
For fixed gains, Equation~\ref{eq:stability-region} is a condition on the
example's sensitivity. The first inequality requires
$(\gamma-1)/K_p<b<(1+\gamma)/K_p$, and the second requires
$b<(1+\gamma)/(K_p+K_i/2)$, which is slightly tighter than the upper bound of
the first. The loop therefore tolerates any positive sensitivity below
$(1+\gamma)/(K_p+K_i/2)$. Since sensitivities differ across examples and are
not known in advance, we use moderate gains and a small integral gain
(Section~\ref{sec:setup}), which widens this range at the cost of a slower
transient. The limit $\alpha_\infty=\alpha^*$ does not depend on the gains;
only the speed of approach does.
\end{remark}

\subsection{Derivative action, saturation, and the commitment horizon}
\label{app:theory-remarks}

\begin{remark}[Derivative action]
\label{rem:derivative}
A derivative term adds $K_d(e_t-e_{t-1})$ to the command. At any fixed point
consecutive errors are equal, so the term vanishes there and leaves the
equilibria of Propositions~\ref{prop:p} and~\ref{prop:pi} unchanged; it only
reshapes the transient, as in the PID analysis of \citet{nguyen2026pid}.
\end{remark}

\begin{remark}[Saturation and anti-windup]
\label{rem:saturation}
Clipping the command at $\alpha_{\max}$ takes the loop outside the linear
analysis. While the command is clipped, the integral state would otherwise
keep growing and later discharge as overshoot; the conditional integration of
Algorithm~\ref{alg:attack}, which freezes $I_t$ whenever the command is
clipped, is the standard guard for this regime and lets the state re-enter
the linear region without that discharge. If $\alpha^*>\alpha_{\max}$, no
admissible command reaches the setpoint under the local model.
\end{remark}

\begin{remark}[Horizon versus bound]
\label{rem:horizon}
Both propositions describe the loop for $t<\tau$. In our runs with
$\alpha_{\max}=6$, the bound is reachable only after the answer token has
committed (Appendix~\ref{app:trajectory}), so for a resistant example the
finite horizon $\tau$ under the chosen gains, rather than $\alpha_{\max}$,
limits the strength it receives. This is why the commitment horizon matters
in Section~\ref{sec:method} and why every gap in the paper is reported with
its invalid-output rate.
\end{remark}

\section{Extended related work}
\label{app:related}

This appendix expands the related-work summary of Section~\ref{sec:related}.

\subsection{Activation steering and its control-theoretic reading}
\label{app:related-steering}
Activation interventions edit a frozen model's hidden states at inference
time, using directions built from prompt pairs
\citep{turner2023actadd,zou2023repe}, averaged contrasts \citep{rimsky2024caa},
head-level edits \citep{li2023iti}, neuron damping \citep{suau2024aura}, or
optimal-transport maps \citep{rodriguez2025act}; a single direction can carry
a behavior such as refusal \citep{arditi2024refusal}. FairSteer applies a
debiasing direction only when a probe on the activations detects bias
\citep{li2025fairsteer}; the Trojan activation attack plants a direction to
induce misaligned outputs \citep{wang2024trojan}, the autoregressive precedent
for our threat model. PID Steering reads the layer-wise construction of a
steering vector as a proportional controller and adds integral and
derivative terms across network depth \citep{nguyen2026pid}. Our loop runs
along a different axis: the same network at successive denoising steps, with
the probability of the attacker's answer as the measured output and a single
scalar strength under control.

\subsection{Inference-time control and attacks on dLLMs}
\label{app:related-dllm}
\citet{shnaidman2026activation} suppress refusal by removing activation
components along a fixed direction throughout reverse diffusion,
\citet{avrahami2026ilrr} align representations with a reference sequence at
each step, and \citet{zhou2026steering} schedule interventions by when
attributes emerge; none uses the evolving probability of a designated answer
as feedback. Adversarial work on dLLMs seeks harmful compliance: DIJA
interleaves masked spans with harmful prompts so that bidirectional infilling
completes them \citep{wen2026devil}, MaskForge automates the search over such
structures \citep{ma2026maskforge}, PAD steers parallel decoding toward
affirmative continuations \citep{zhang2025pad}, \citet{he2026safer} nest
harmful requests in benign context, \citet{singh2026remask} re-mask committed
refusal tokens, and \citet{dumitrescu2026targets} prune safety-related
neurons; DiffuGuard defends the denoising loop against such attacks
\citep{li2026diffuguard}. TrustLDM audits dLLMs for bias and shows that
malicious context amplifies it \citep{mo2026trustldm}. We instead induce a
targeted demographic preference through activations, without touching the
prompt, weights, or commitment schedule.

\subsection{Measuring demographic bias}
\label{app:related-bias}
Underspecified questions expose reliance on stereotypes when abstention is
the correct answer \citep{li2020unqover,parrish2022bbq}, and SocialStigmaQA
tests whether generation amplifies documented stigmas
\citep{nagireddy2024socialstigmaqa}; broader benchmarks cover stereotypes in
pretrained models and in generated text
\citep{nadeem2021stereoset,jin2025social,wan2025white}. We repurpose BBQ and
SocialStigmaQA as attack objectives and inherit their pitfalls: answer
position, abstention handling, invalid generations, and response parsing all
change the apparent effect, and Section~\ref{sec:experiments} controls for
each.

\end{document}